%% file: copco_paper_arxiv.tex
\documentclass[a4paper,twocolumn]{article}

\usepackage[utf8]{inputenc}
\usepackage[a4paper,textwidth=506.295pt,centering]{geometry}
\usepackage{amssymb}
\usepackage{amsmath}
\usepackage{amsthm}
\usepackage{booktabs}
\usepackage{mathrsfs}
\usepackage{verbatim}
\usepackage{lscape}
\usepackage{mathtools}
\usepackage{float}
\usepackage{enumitem}
\usepackage[justification=centering]{caption}
\usepackage{subcaption}
\usepackage{longtable}
\usepackage{xspace}
\usepackage[ruled]{algorithm2e}
\SetAlgoSkip{bigskip}
\SetKwInput{KwData}{Input}
\SetKwInput{KwResult}{Output}
\SetKwComment{Comment}{// }{}
\usepackage{tikz-cd}
\usepackage{thmtools}
\usepackage{hyperref}
\usepackage{multirow}
\usepackage{array}
\newcolumntype{C}{>{\(}c<{\)}} 
\usepackage[numbers,sort&compress]{natbib}
\setcitestyle{numbers,square}
\let\citet\citep

\let\citealt\citep

\theoremstyle{plain}
\newtheorem{theorem}{Theorem}[section]
\newtheorem{corollary}[theorem]{Corollary}
\newtheorem{proposition}[theorem]{Proposition}

\theoremstyle{definition}
\newtheorem{definition}[theorem]{Definition}
\newtheorem{example}[theorem]{Example}

\theoremstyle{remark}
\newtheorem{remark}[theorem]{Remark}

\newcommand{\ZZ}{\mathbb{Z}}

\newcommand{\RR}{\mathbb{R}}

\newcommand{\Mid}{\,\middle|\,}

\graphicspath{{figures/}}

\title{Fixation Sequences as Time Series: A Topological Approach to Dyslexia Detection}

\author{Marius Huber\thanks{Supported by the Swiss National Science Foundation (grant no. 209413).} \quad David R. Reich\thanks{Supported by the Swiss National Science Foundation (grant IZCOZ0\_220330/1 (EyeNLG))} \quad Lena A. J{\"a}ger
\vspace{0.1cm}\\
        Department of Computational Linguistics,
        University of Zürich, Switzerland \\
        \texttt{\{marius.huber,davidrobert.reich,lenaann.jaeger\}@uzh.ch}
}

\date{}

\begin{document}

\maketitle

\begin{abstract}
  Persistent homology, a method from topological data analysis, extracts robust, multi-scale features from data.
  It produces stable representations of time series by applying varying thresholds to their values (a process known as a \textit{filtration}).
  We develop novel filtrations for time series and introduce topological methods for the analysis of eye-tracking data, by interpreting fixation sequences as time series, and constructing ``hybrid models'' that combine topological features with traditional statistical features.
  We empirically evaluate our method by applying it to the task of dyslexia detection from eye-tracking-while-reading data using the Copenhagen Corpus, which contains scanpaths from dyslexic and non-dyslexic L1 and L2 readers.\footnote{
    All code to run experiments is available at \url{https://github.com/DiLi-Lab/topological-dyslexia-detection} and archived at \url{https://doi.org/10.5281/zenodo.19238279}.
  }
  Our hybrid models outperform existing approaches that rely solely on traditional features, showing that persistent homology captures complementary information encoded in fixation sequences.
  The strength of these topological features is further underscored by their achieving performance comparable to established baseline methods.
  Importantly, our proposed filtrations outperform existing ones.
\end{abstract}


\section{Introduction}\label{sec:introduction}

Topological data analysis (TDA) has emerged as a powerful framework for extracting features from complex data sets that capture not only more traditional, statistics-based properties of a data set, but also its shape.
These features are robust to noise and coordinate-invariant, and useful in data analysis, visualization and machine learning tasks~\cite{chazal_intro_to_TDA,Carlsson_topology_and_data}.
In recent years, persistent homology, a key tool within TDA, has been increasingly applied to a wide variety of data sets (for accessible introductions, see~\citet{Carlsson_Vejdemo-Johansson,ghrist_barcodes}).
Informally, computation of persistent homology is obtained from a so-called filtration on a data set, which can be thought of as a way of ``sweeping'' the data set with respect to some relevant feature.
Persistent homology tracks topological features (the ``shape'') of a data set across various stages of a filtration, producing a provably stable representation of the topology and geometry of the underlying data set~\cite{cohen_steiner_stability_of_persistence_diagrams}.

If the data set at hand is a time series, that is, a sequence of values indexed by time, persistent homology may be computed from a filtration that is defined on the time series itself, and it enables the characterization of geometric and dynamical structures intrinsic to temporal signals~\cite{gidea_katz_tda_financial_time_series,perea_sliding_windows_persistence,seversky_tda_time_series}.
In this context, a filtration may be defined, for example, by thresholding the values the time series attains, and gradually letting that threshold sweep the entirety of the time series.
For an illustration, see Figure~\ref{fig:time_series_example}, which shows a time series (in red) being swept by a horizontal line (in blue); see Example~\ref{exp:main_example} for a full discussion of this.
The resulting persistent homology can thus capture properties of the time series such as its amplitude and periodic behavior.

The main contribution of this work is twofold: first, we extend existing topological methods for time series classification, and second, we introduce the above framework to eye-tracking data and empirically evaluate whether TDA-based features add additional information that is not captured by traditional statistical or hand-crafted features.
In existing work, persistent homology of a time series is usually computed with respect to a so-called horizontal filtration that sweeps the time series by horizontal lines of increasing heights~\cite{ost_banana_trees,montesano_dynamically_maintaining_ph_of_time_series}.
In particular, this horizontal filtration is agnostic to the temporal coordinate of the time series.
In order to capture the temporal information contained in a time series, we introduce non-horizontal filtrations which explicitly depend on the temporal coordinate.
In the case where the time series is a fixation sequence, these non-horizontal filtrations enable sensitivity to, for instance, the total number of fixations in a sequence, the time elapsed between consecutive fixations, or the presence and number of regressions.
Second, and as a case study introducing persistent homology to eye-tracking data, we consider the task of dyslexia detection from reading data.
Concretely, we extend existing models for dyslexia detection to contain the TDA-based features, in order to investigate whether persistent homology can provide features which are complementary to more traditional, hand-crafted ones.

To extract topological features from a fixation sequence of an individual reading a text, we first translate the information contained in it into a pair of time series: one containing the horizontal coordinates of the fixations indexed by the time of their onset, and the other containing the vertical coordinates.
This separation of horizontal and vertical movements is done because our topological approach is suitable only for time series where the values are numbers (as opposed to vectors).
In the context of eye movements, however, this separation also enables persistent homology to capture reading behavior constrained to a single line (from the horizontal movements) as well as that happening across lines (from vertical movements); this is particularly useful when dealing with fixation sequences stemming from reading of multiline stimuli, as is the case in the present article.
From these two time series, we compute persistent homology, and combine the resulting topological features with the traditional ones obtained from existing models for dyslexia detection.
We find that when adding these topological features, the resulting models outperform their counterparts that use only traditional features.
Moreover, we find that even when removing all traditional features, the topological features on their own are sufficient to achieve near-state-of-the-art classification performance in dyslexia detection.
This suggests that persistent homology captures aspects of fixation sequences that are complementary to those captured by traditional methods.
For the above, we train and evaluate various pipelines on the Copenhagen Corpus of Eye Tracking Recordings from Natural Reading of Danish Texts (CopCo)~\cite{hollenstein-etal-2022-copenhagen,bjornsdottir2023dyslexia,reich-etal-2024-reading}, a large eye-tracking data set specifically designed for dyslexia research.


\section{Background and Related Work}\label{sec:background}

We now review how persistent homology may be used for time series classification, and then provide relevant background on dyslexia detection.


\subsection{Persistent Homology for Time Series Classification}\label{sec:background_persistent_homology_for_time_series_classification}

We provide a conceptual review of persistent homology in time series classification, as described, for instance, in~\citet{ost_banana_trees} or \citet{montesano_dynamically_maintaining_ph_of_time_series}.
For a more technical discussion, we refer to Appendix~\ref{appendix:persistent_homology_of_time_series}.

In the following, we use the term \emph{time series} to refer to a finite sequence \(\left((t_{1},x_{1}),\dots,(t_{n},x_{n})\right)\) of \(n\) pairs of real numbers, where the \(t_{1}<t_{2}<\cdots<t_{n}\) represent points in time, and \(x_{1},x_{2},\dots,x_{n}\) denote the values that the time series attains at the respective points in time.
In the context of eye-tracking data, for example, a fixation sequence naturally gives rise to a time series where the \(t\)-values represent the onset of a fixation and the \(x\)-values represent either the horizontal or vertical coordinate of the corresponding fixation.
The former is expected to be more sensitive to eye movements within a single line than to those across lines, and vice versa for the latter.\footnote{
  One could generalize the notion of time series to one that captures both the horizontal and the vertical coordinate of a fixation.
  This, however, would yield a so-called multivariate time series, a format that is unsuitable for the methods of time series classification outlined here.
  Note also that with time series stemming from other domains such as finance or audio signals, we would not necessarily expect one coordinate to be more informative than others.
}
In our experiments, we therefore work with both time series extracted from a single fixation sequence (see Section~\ref{sec:pipeline} for details).

Persistent homology~\cite{ghrist_elementary_applied_topology,edelsbrunner_computational_topology} is a topological tool that, broadly speaking, ``sweeps'' a time series and produces a so-called persistence diagram that captures characterizing features of the time series.
One approach to this is to represent a time series \((t_{1},x_{1}),\dots,(t_{n},x_{n})\) by the graph in the \((t, x)\)-plane that is obtained by drawing a straight line connecting the point \((t_{1},x_{1})\) with \((t_{2},x_{2})\), one connecting \((t_{2},x_{2})\) with \((t_{3},x_{3})\), and so on.
Persistent homology is obtained by sweeping this graph by a horizontal line moving from bottom to top, and recording the evolution of the number of line segments that constitute the portion of the graph lying below the horizontal line.

\input{tables/table_fixation_sequence.tex}

\begin{figure*}
  \centering
  \def\svgscale{1}
  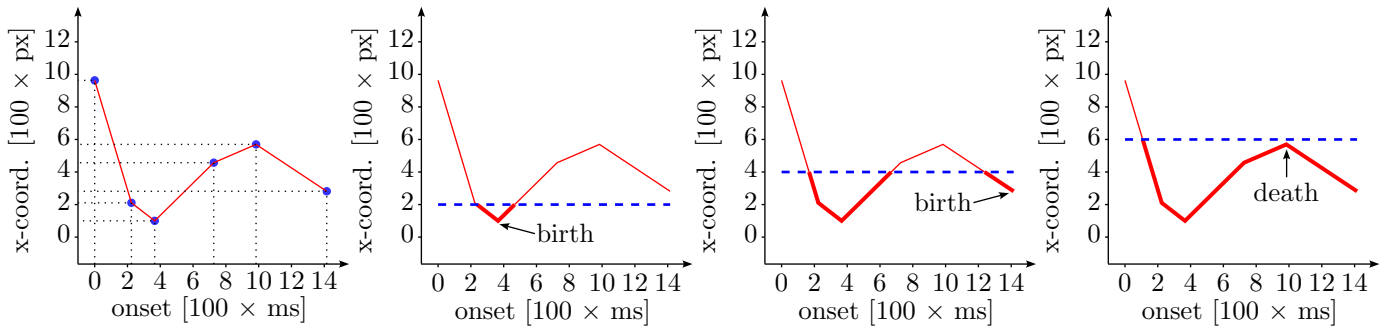
  \caption{Graph corresponding to the time series from Example~\ref{exp:main_example} (leftmost panel) and three stages of sweeping of the graph by a horizontal line for increasing heights (second through fourth panels).
    The number of line segments constituting the portion of the graph lying under the horizontal line changes from one to two and back to one as the height of the horizontal line changes from \(x=200.0\) to \(x=400.0\) and to \(x=600.0\), respectively.}
  \label{fig:time_series_example}
\end{figure*}

\begin{example}[Sweeping the time series corresponding to a fixation sequence by horizontal lines]\label{exp:main_example}
  Consider the horizontal movements of the fixation sequence given in Table~\ref{table:fixation_sequence} (a procedure analogous to the following one applies to the vertical movements).
  The corresponding time series is specified by the sequence \[
    \begin{aligned}
      \left((t_{1},x_{1}),\dots,(t_{6},x_{6})\right)=( & (0,963.1), (244,210.7),      \\
                                                       & (366,100.0), (726,457.2),    \\
                                                       & (984,569.6), (1415, 281.6)).
    \end{aligned}
  \]
  Its graph is obtained by connecting consecutive \((t,x)\)-pairs by straight lines; it is pictured in the leftmost panel of Figure~\ref{fig:time_series_example}, where fixations are indicated by blue dots, and dotted lines indicate the \(t\)- and \(x\)-entries of the pairs that constitute the time series.
  Sweeping this graph starts out with a horizontal line that is situated at a height that is theoretically infinitely far below the graph.
  For a horizontal line placed at any height that is smaller than the global minimum of the time series, the portion of the graph lying under that line is empty.
  As the height of the horizontal line passes the global minimum value \(x=100.0\), this changes: for a horizontal line at height, for example, \(x=200.0\), the portion of the graph lying below it now consists of a line segment near the point \((366,100.0)\) in the \((t,x)\)-plane; see the second panel from the left of Figure~\ref{fig:time_series_example}, where the horizontal line at height \(x=200.0\) is indicated in blue, and where the portion of the graph lying below it is indicated by the thickened segment.
  As the horizontal line moves further upwards, we see that a new line segment is ``born'' as the line passes the value \(x=281.6\): at height, for example, \(x=400.0\) the portion of the graph lying under the horizontal line now consists of two separate line segments; see the second panel from the right of Figure~\ref{fig:time_series_example}.
  This newly born segment ``dies'' as the height of the horizontal line passes \(x=569.6\), since it gets absorbed into the line segment that was born earlier at height \(x=100.0\).\footnote{
    One defines that ``older'' connected line segments absorb ``younger'' ones (and not vice versa) for technical reasons; without this definition, computation of persistent homology would not lead to a well-defined result~\cite[Chapter VII.1]{edelsbrunner_computational_topology}.
  }
  Correspondingly, the number of segments is reduced back to one at height, for example, \(x=600.0\); see the rightmost panel of Figure~\ref{fig:time_series_example}.
  Note that this behavior is a consequence of the fact that the time series returns to a value near \(x=200\) toward its end (indicating a within-line regression of the reader).
  Indeed, if no within-line regressions were present, the graph in Figure~\ref{fig:time_series_example} would lack the local maximum responsible for the death event.
\end{example}

Sweeping the graph corresponding to a time series thus yields a sequence of the number of segments present at each height of the horizontal line.
This information is recorded in a \emph{persistence diagram}~\cite{edelsbrunner_topological_persistence}; we illustrate this concept in the following example, and refer to Appendix~\ref{appendix:persistent_homology_of_time_series} for details.

\begin{figure}
  \centering
  \def\svgscale{1}
  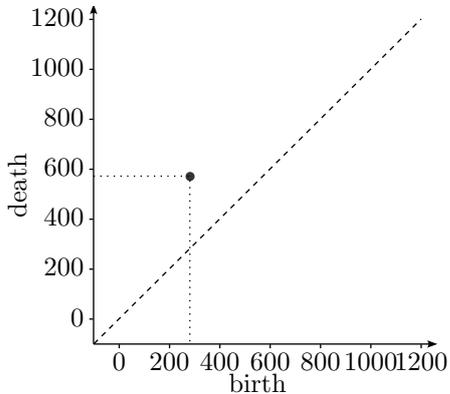
  \caption{Persistence diagram corresponding to the time series from Example~\ref{exp:main_example}, obtained from sweeping by horizontal lines.
    The point located at \((281.6,569.6)\) records the line segment that is born at \(x=281.6\) and dies at \(x=569.6\).}
  \label{fig:persistence_diagram_example}
\end{figure}

\begin{example}[Persistence diagram]\label{exp:ordinary_persistence_diagram}
  Consider the time series given in Example~\ref{exp:main_example}.
  By definition, a persistence diagram is obtained by placing a point at the coordinates \((b,d)\) in \(\RR^{2}\), where \(b\) and \(d\) record the height of the horizontal line at which a line segment is born and dies, respectively.
  In Example~\ref{exp:main_example}, the first line segment is born at height \(x=100.0\), whereas the second one is born at \(x=281.6\) and dies at \(x=569.6\).
  While the former never dies and thus does not have a death time, the birth and death times of the latter are recorded by including the point \((281.6,569.6)\) in the persistence diagram; see Figure~\ref{fig:persistence_diagram_example}.
  Additionally, a persistence diagram is endowed with the diagonal line which indicates the region of the persistence diagram where birth time equals death time; the greater the vertical distance of a point corresponding to a line segment to this diagonal line, the greater the difference between its death and birth time, and hence the longer this segment survives for.
  Line segments surviving for a long time are deemed to represent more ``significant'' features of the time series than short-lived ones.
  A persistence diagram thus provides a summary of the evolution of the line segments as one sweeps a time series by horizontal lines.
\end{example}

We point out that persistence diagrams computed from a time series like above typically contain many more points than that pictured in Figure~\ref{fig:persistence_diagram_example}.
Indeed, the ``sparsity'' of that persistence diagram is a consequence of the fact that the underlying time series is very short.

\begin{remark}\label{rem:extended_persistence}
  The process above produces one persistence diagram from a time series, and is known as \emph{ordinary persistent homology}.
  In our experiments, we also use \emph{extended persistent homology}~\cite{cohen_steiner_edelsbrunner_harer_extended_persistence}.
  In a nutshell, this is obtained by first sweeping the graph corresponding to a time series by horizontal lines as above, and then sweeping the same graph in the opposite direction---from top to bottom---to capture additional information.
  We point out that extended persistent homology produces three persistence diagrams from a time series, and refer to Appendix~\ref{appendix:persistent_homology_of_time_series} for details.
\end{remark}

\begin{figure*}
  \centering
  \def\svgscale{1}
  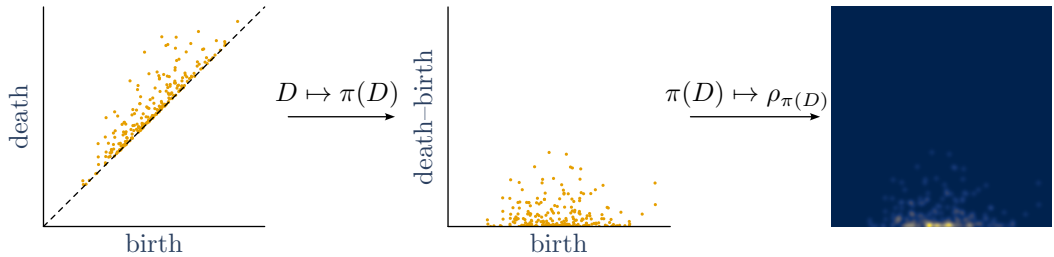
  \caption{Schematic showing the transformation of a persistence diagram \(D\) in (birth, death)-coordinates (left) to (birth, death-birth)-coordinates (middle) by means of a shearing \(\pi\), and the resulting persistence surface \(\rho_{\pi(D)}\) (right).}
  \label{fig:persistence_image}
\end{figure*}

In order to make persistence diagrams usable in machine learning pipelines, one must transform them into vectors of a fixed length.
Methods for doing so abound~\cite{di_fabio_cpx_polynomials,bubenik_persistence_landscape,chazal_silhouette,carriere_topological_vector}, and we choose to use the method known as \emph{persistence image}~\cite{adams_persistence_images}.
In a nutshell, a persistence image is obtained from a persistence diagram by first applying a shearing that maps the diagonal of the persistence diagram to a horizontal line.
The resulting diagram is then converted into a ``heat map'' by placing a bivariate Gaussian distribution with a given standard deviation (also referred to as \emph{bandwidth} in the present context) at each point of the diagram.
This yields a single-channel image commonly referred to as persistence surface; see Figure~\ref{fig:persistence_image} for an illustration of the process of obtaining a persistence surface from a persistence diagram.
Finally, the persistence image is obtained from the persistence surface by discretizing it into a pixel grid of a \emph{resolution} prescribed by the user.
We refer to Appendix~\ref{appendix:persistent_homology_and_ml} for details on persistence images, and conclude this section by pointing out that a persistence image created as above depends on the bandwidth and resolution used; in applications, these are treated as hyperparameters.


\subsection{Predicting Dyslexia from Eye-Tracking-While-Reading Data}\label{sec:background_eyetracking}

Dyslexia is one of the most common learning disabilities that affects reading, impacting an estimated 9–12\% of the population~\cite{katusic2001incidence, shaywitz1998functional}.
For children with dyslexia, identifying the condition early is crucial for helping them succeed and keep pace in school~\cite{glazzard2010impact, torgesen2000individual, vellutino2004specific}.
Eye-tracking technology has emerged as a promising tool in this context.
It offers a good balance between cost, invasiveness and data quality, and as the technology continues to improve and become more affordable, it is increasingly realistic to use it on a large scale.
A single eye tracker in a school could serve multiple screening purposes, for example, evaluating reading proficiency in a second language~\cite{berzak2018assessing}.
These developments raise the question of whether dyslexia can be detected from eye movements.
Researchers have been exploring links between dyslexia and eye-movement behavior since the 1980s~\cite{olson1983dyslexic, pavlidis1981eye, rayner1985faulty}.
While early studies focused on deficits in eye-movement control, later research shifted attention toward difficulties in phonological decoding as the core cause of dyslexia.
Importantly, eye movements are closely tied to phonological decoding~\cite{IDA2026} throughout life~\cite{rayner1995phonological, rayner1998phonological, blythe2014developmental, leinenger2019survival, milledge2019changing}, which provides a strong theoretical basis for using eye-movement data to infer dyslexia.
Indeed, a growing body of experimental research~\cite{rello2015detecting, nilsson2016screening, asvestopoulou2019dyslexml, raatikainen2021detection, haller2022eye, jothi2022prediction, shalileh2023identifying, bjornsdottir2023dyslexia, laurinavichyute2025automatic} shows that machine learning approaches can effectively screen for dyslexia using features derived from eye movements.
Here, we use the work of \citet{bjornsdottir2023dyslexia} and \citet{raatikainen2021detection} as reference methods.


\section{Problem Setting}\label{sec:problem_setting}

We empirically study the usefulness of casting fixation sequences as time series, and combining topological features computed from them with traditional eye-movement features.
This is to investigate whether the former capture information complementary to that captured by the latter.
We do this for the problem of dyslexia detection from eye-tracking-while-reading data for two levels of aggregation: trial-level and reader-level.

For each individual \(j\), their reading behavior during the reading of stimulus text \(k\) is recorded as a sequence \[
  T_{j,k}=\left((t_{1},x_{1},y_{1}),\ldots,(t_{M_{j,k}},x_{M_{j,k}},y_{M_{j,k}})\right)
\] of \(M_{j,k}\) fixations, where \(t_{m}\) denotes the onset of the \(m\)-th fixation and \((x_{m},y_{m})\) denotes its location.
In our training data set \(\mathcal{D}=\left\{(T_{j,k}, c_{j})\right\}_{j,k}\), each fixation sequence \(T_{j,k}\) is associated with a binary label \(c_{j}\in\{0, 1\}\), indicating the absence (\(c_{j}=0\)) or diagnosis (\(c_{j}=1\)) of dyslexia in the reader.
In trial-level aggregation, we aim to model the conditional probability \(P\left(c_{j}\Mid T_{j,k}\right)\), whereas in reader-level aggregation, we aim to model \(P\left(c_{j}\Mid \left\{T_{j,k}\right\}_{k}\right)\).
The performance of our classification models will be evaluated using the Area Under the Receiver Operating Characteristic curve (ROC AUC).
The ROC AUC provides a robust measure of discriminatory power, as it considers the trade-off between true positive and false positive rates across various classification thresholds.
In particular, ROC AUC is less sensitive to class imbalance than accuracy, which is relevant in our case given the unequal distribution of dyslexic and non-dyslexic participants in the data set we used (see Section~\ref{sec:data_set}).


\section{Method}\label{sec:method}

We inject our topological features extracted from fixation sequences into the two models introduced in~\citet{bjornsdottir2023dyslexia} and~\citet{raatikainen2021detection} (the ``baseline models'').
We compare the resulting ``hybrid models'' trained with these features against the two baseline models, and against models that take into account only the topological features (the ``TSH-models'', where TSH stands for ``time series homology'').
We perform this for both trial-level and reader-level aggregation, and including/excluding data stemming from non-native readers, which results in four experiment settings (here, a trial refers to a participant's reading of one stimulus text; see Section~\ref{sec:data_set} for details).

We now introduce an extension of the sweeping method discussed in Section~\ref{sec:background_persistent_homology_for_time_series_classification}, and then present our dyslexia detection pipeline for each of the three model classes mentioned above.


\subsection{Computation of Topological Features and Non-Horizontal Filtrations}\label{sec:computation_of_topological_features_and_non_horizontal_filtrations}

To compute the topological features of a time series, we use the horizontal filtration as described in Section~\ref{sec:background_persistent_homology_for_time_series_classification}.\footnote{
  The term ``filtration'' is simply a more technical term for a sweeping method.
}
A shortcoming of that method is that this process is agnostic to the time coordinate of the time series.
Indeed, the birth and death times of line segments are dependent only on the values of the local extrema of the time series, but not on the point in time at which these local extrema occur.\footnote{\label{fn:horizontal_is_agnostic_to_time}
  As an example of this, observe that inverting the time coordinate of a time series and computing persistent homology using the horizontal filtration yields the same result as computing persistent homology of the original time series.
}
We therefore introduce three new filtrations, which we call \emph{sloped}, \emph{sigmoid} and \emph{arctan filtrations}.

We illustrate the sloped filtration with an example, and refer the reader to Appendix~\ref{appendix:non_horizontal_filtrations} for technicalities, and details on the sigmoid and arctan filtrations.

\begin{figure}
  \centering
  \def\svgscale{1}
  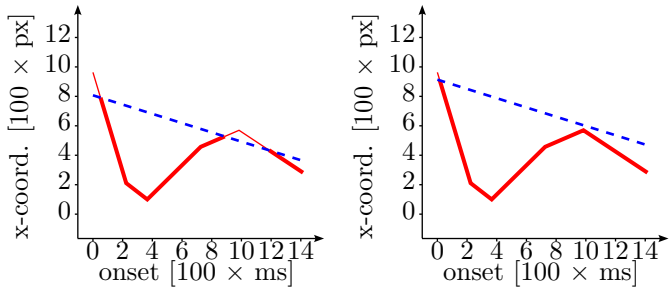
  \caption{Two stages of sweeping of the graph corresponding to a time series by a line of slope \(c\approx-0.3\) moving from left to right (indicated in blue).
    The number of line segments constituting the portion of the graph lying to the left of the line changes from two to one as the line moves from the left panel to the right one.}
  \label{fig:time_series_example_sloped}
\end{figure}

\begin{example}[Sloped filtration of a time series]\label{exp:sloped_filtration}
  Consider the time series from Example~\ref{exp:main_example}, whose corresponding graph is pictured in the left of Figure~\ref{fig:time_series_example}.
  To start, one places a line with slope \(c\neq 0\) far enough to the left in the \((t,x)\)-plane so that the graph corresponding to the time series lies entirely to the right of that line (in experiments, \(c\) is a tunable hyperparameter).
  One then lets the sloped line move horizontally to the right while keeping track of the evolution of the number of line segments that constitute the portion of the graph that lies to the left side of the line.
  As in the case of the horizontal filtration, the birth and death times of the line segments being born and dying along the way are recorded in a persistence diagram.
  For an illustration of this, see Figure~\ref{fig:time_series_example_sloped}.
  In that figure, the graph corresponding to the time series is swept by a line of slope \(c\approx-0.3\) moving from left to right (indicated in blue), and the portion of the graph lying to the left of that line is indicated by the thickened segments.
\end{example}

Crucially, the sloped filtration is no longer agnostic to the time coordinate.\footnote{
  Indeed, inverting the time coordinate of the time series and using the sloped filtration will result in a persistence diagram that is different from the one computed from the non-inverted time series; cf. Footnote~\ref{fn:horizontal_is_agnostic_to_time}.
}
For eye movements, this makes the filtration sensitive to, for instance, the total number of fixations.
For the sigmoid and arctan filtrations, one replaces the sloped line by a curve described by the logistic and the arctan functions, respectively, and lets that curve move from the left to the right while keeping track of the number of line segments lying to the left of the curve, as before.
These filtrations can ``see'' backward and forward in the time coordinate thanks to their curved shape, thus capturing, for instance, the time elapsed between two fixations.


\subsection{Pipeline for Dyslexia Detection}\label{sec:pipeline}

Given a fixation sequence \(T_{j,k}=\left\{(t_{i}, x_{i}, y_{i})\right\}_{i=1}^{M_{j,k}}\) stemming from the \(k\)-th trial of individual \(j\) and consisting of \(M_{j,k}\) fixations, we wish to predict the binary label \(c_{j}\in\left\{0, 1\right\}\), indicating the absence or diagnosis of dyslexia.

\begin{figure*}
  \centering
  \def\svgscale{1}
  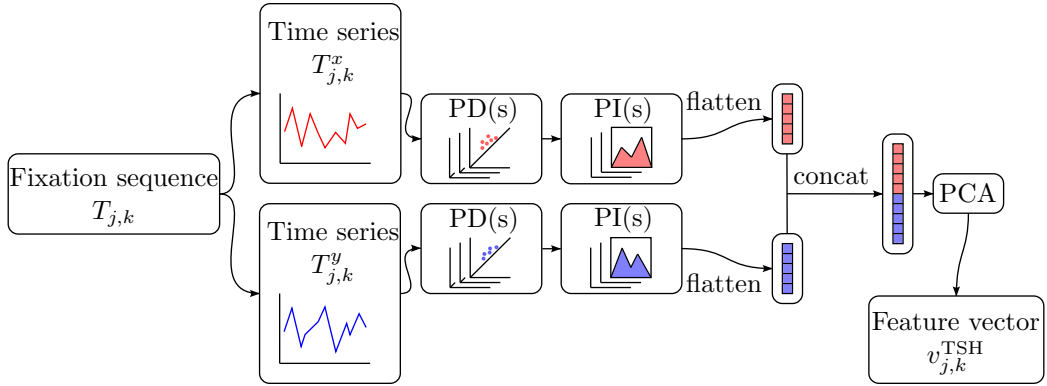
  \caption{Computation of the topological feature vector \(v_{j,k}^{\mathrm{TSH}}\) from a fixation sequence.
    A fixation sequence \(T_{j,k}\) is split up into two time series \(T_{j,k}^{x}\) and \(T_{j,k}^{y}\), after which the persistence diagrams (PDs) and the corresponding persistence images (PIs) are computed for each.
    The persistence images are each flattened, concatenated and reduced via PCA to yield the feature vector \(v_{j,k}^{\mathrm{TSH}}\).
  }
  \label{fig:pipeline_schematic}
\end{figure*}

For the hybrid models combining traditional features with topological ones, we first compute the feature vector \(v_{j,k}^{\mathrm{BL}}\) from \(T_{j,k}\) via the relevant baseline model.
For the topological feature vector, we split up the fixation sequence \(T_{j,k}\) into the sequences \(T_{j,k}^{x}\coloneqq\left\{(t_{i}, x_{i})\right\}_{i=1}^{M_{j,k}}\) and \(T_{j,k}^{y}\coloneqq\left\{(t_{i}, y_{i})\right\}_{i=1}^{M_{j,k}}\) containing just the \(x\)- and \(y\)-coordinates of the fixation sequence, respectively.
We interpret each of \(T_{j,k}^{x}\) and \(T_{j,k}^{y}\) as a time series, min-max-scale their values, and compute persistent homology with respect to either the horizontal, sloped, sigmoid or arctan filtration; for the latter three filtrations, we treat the slope \(c\) as a tunable hyperparameter.
This yields one or three persistence diagrams (PDs) for each of \(T_{j,k}^{x}\) and \(T_{j,k}^{y}\), depending on whether ordinary or extended persistence is computed (cf. Remark~\ref{rem:extended_persistence}).
This results in two or six diagrams, respectively, for each initial fixation sequence \(T_{j,k}\).
Each persistence diagram is vectorized into a persistence image (PI) of a fixed resolution.
The persistence images are flattened, min-max-scaled and concatenated into a single vector.
Using a resolution of \((n,m)\) for the persistence images, this resulting vector has length \(2\cdot 1\cdot n\cdot m=2nm\) if ordinary persistence is used, and length \(2\cdot 3\cdot n\cdot m=6nm\) if extended persistence is used; in our experiments, we set \(n=m=50\) throughout.
Pixel values of persistence images are highly correlated, so we reduce this vector using PCA to obtain the topological feature vector \(v_{j,k}^{\mathrm{TSH}}\) corresponding to the fixation sequence \(T_{j,k}\); in our experiments, we use PCA with \(250\) and \(3\cdot250=750\) components for ordinary and extended persistence, respectively, as this number of components is empirically sufficient to capture the information stored in the pixel values.
For a schematic depiction of the computation of \(v_{j,k}^{\mathrm{TSH}}\), see Figure~\ref{fig:pipeline_schematic}.
Finally, we concatenate \(v_{j,k}^{\mathrm{BL}}\) and \(v_{j,k}^{\mathrm{TSH}}\) into a single feature vector \(v_{j,k}\), containing both traditional and topological features of \(T_{j,k}\).
For trial-level aggregation, we pass \(v_{j,k}\) into the classifier specified by the baseline model for prediction of \(c_{j,k}\in\left\{0, 1\right\}\); for the baseline introduced in~\citet{bjornsdottir2023dyslexia}, this is a support vector machine (SVM), whereas for the one introduced in~\citet{raatikainen2021detection}, both a SVM and a random forest (RF) are used.
For reader-level aggregation, we compute the mean of all vectors \(v_{j,1},v_{j,2},\dots\) stemming from individual \(j\) before passing the resulting vector into the respective classifier for prediction of \(c_{j}\in\left\{0, 1\right\}\).

For the baseline and TSH-models, we pass just the feature vector \(v_{j,k}^{\mathrm{BL}}\) and \(v_{j,k}^{\mathrm{TSH}}\), respectively, into the classifier; the latter is either a SVM or a RF for the baseline models, and a SVM for the TSH-models.


\section{Experiments}\label{sec:experiments}

In the following, we describe the data set and evaluation protocol used in our experiments; all code to run them is available at \url{https://github.com/DiLi-Lab/topological-dyslexia-detection} and archived at \url{https://doi.org/10.5281/zenodo.19238279}.


\subsection{Data Set}\label{sec:data_set}

The Copenhagen Corpus of Eye-Tracking Recordings from Natural Reading (CopCo; \citealt{hollenstein-etal-2022-copenhagen}) is a corpus of eye-tracking data collected during natural reading of Danish texts.
These consisted of 46 transcriptions of Danish speeches and 12 articles from Danish Wikipedia, which were presented to participants in multiline portions containing at most \(10\) lines.
We refer to a participant's reading of one such portion as a trial.
Participants read a varying number of texts according to their own speed, and reading data was collected tracking the right eye, using an EyeLink 1000 Plus eye tracker with a sampling rate of \(1000\) Hz.
The data set comprises recordings from 58 participants, including 22 L1 (native) speakers of Danish not diagnosed with dyslexia, 19 L1 speakers diagnosed with dyslexia, and 17 L2 speakers of Danish without dyslexia.\footnote{
  The non-L1 speakers stem from different L1 backgrounds.
  Moreover, as in~\citet{reich-etal-2024-reading}, we refer to the non-L1 speakers as L2 speakers for simplicity, even though Danish is L3 or L4 for some of the participants.
}
We discarded the data from four out of the 22 L1 readers without dyslexia due to poor calibration or a diagnosed attention disorder.
Moreover, the data from one of the L1 readers with dyslexia was discarded due to the absence of a dyslexia screening result, leaving data from a total of 53 readers.
The fixation sequences used in our models were taken directly from CopCo.
These were previously extracted using the DataViewer software by SR Research, which applies a velocity- and acceleration-based saccade detection method~\citep{hollenstein-etal-2022-copenhagen}.
We discarded those trials consisting of fewer than five fixations.
Each participant was assigned a binary label based on the original CopCo annotations: participants diagnosed with dyslexia were labeled as ``dyslexic'', and all others as ``non-dyslexic''.
This results in a data set consisting of 4653 fixation sequences (out of which 1134 or 24.37\% are labeled as ``dyslexic'') on the trial-level, and consisting of \(53\) readers (out of which 18 or 33.96\% are labeled as ``dyslexic'') on the reader-level.


\subsection{Evaluation Protocol}\label{sec:evaluation_protocol}

We train the pipeline introduced in Section~\ref{sec:pipeline} for each of the hybrid, baseline and TSH-models, and in each of the settings specified by aggregation level and inclusion/exclusion of non-L1 readers.
We evaluate the pipelines via \(k\)-fold nested cross-validation, where \(k=5\) for trial-level aggregation and \(k=10\) for reader-level aggregation, using Area Under the Receiver Operating Characteristic curve (ROC AUC) as the scoring function.
In both the trial-level and reader-level aggregation, the folds are stratified according to dyslexia label.
In the trial-level aggregation, the folds are additionally such that data from a given reader appears in precisely one fold, ensuring that the model always makes predictions about unseen readers.

\subsubsection{Hyperparameter Tuning}\label{sec:hyperparameter_tuning}

For the hybrid models, the pipeline described in Section~\ref{sec:pipeline} depends on the slope parameter \(c\neq 0\) specifying the filtration, the bandwidth \(\sigma>0\) of the persistence image, and the parameters of the classifier stemming from the respective baseline model; we tune these via randomized search according to the distributions for the slope and bandwidth given in Table~\ref{table:hyperparameter_search}, and the distributions for the parameters of the classifier as given in the respective original implementation.
The baseline models depend only on the parameters of their respective classifiers; we tune these according to their original implementations.
For the TSH-models, the pipeline depends on the slope parameter \(c\neq 0\), the bandwidth \(\sigma>0\), and on the hyperparameters specifying the SVM classifier; we tune these via randomized search according to the distributions given in Table~\ref{table:hyperparameter_search}.
Note that we excluded values between \(-0.5\) and \(0.5\) for the slope \(c\) for computational stability.
In all instances of randomized search, \(200\) iterations are performed.

\input{tables/table_hyperparameters.tex}


\subsection{Results}\label{sec:results}

In Table~\ref{table:results_roc_auc_short}, we report the classification performance of the model classes described in Section~\ref{sec:method}.

\input{tables/table_results_roc_auc_short.tex}

For each model class, we select the best-performing model and report its mean and standard deviation of the ROC AUC across all folds, and we highlight in bold the best performance in each of the four settings resulting from including/excluding L2-readers and aggregating on trial/reader-level (for a detailed breakdown of performances, see Table~\ref{table:results_roc_auc} in Appendix~\ref{appendix:experiment_results}).
The results show that the hybrid models consistently outperform the baseline and the TSH-models, and, moreover, that the TSH-models perform on par with the baseline models, or come close to doing so.
Finally, the results demonstrate that our novel, non-horizontal filtrations mostly outperform the existing horizontal one.


\section{Discussion}\label{sec:discussion}

This study assessed the value of introducing novel filtrations for time series and of using these filtrations to extract topological features from fixation sequences with persistent homology.
We provide strong empirical evidence that the proposed topological features encode information about reading behavior and are complementary to traditional features: across all experimental settings (trial-/reader-level and including/excluding non-native readers), the hybrid models that combine topological and traditional features consistently achieved the highest ROC AUC scores.
Moreover, seeing that each of the highest-scoring models makes use of one of the newly introduced filtrations, we empirically showed that these provide improvements over existing filtrations.
The competitive performance of the TSH-models further indicates that topological representations of fixation time series are themselves highly informative.
Despite abstracting away from explicit linguistic alignment and relying only on geometric and temporal structure, these models perform on par with established baselines.
This suggests that dyslexia-related differences manifest not only in more traditional, aggregate reading measures, but also in higher-order temporal and structural patterns of eye movements that are captured by persistent homology.
The introduction of non-horizontal filtrations appears particularly relevant, as they explicitly incorporate temporal ordering and spacing of fixations, properties that are central to reading dynamics.

Performance gains are more pronounced at the reader level, which is expected given the aggregation of multiple trials into a more stable behavioral profile.
At the same time, the very high reader-level ROC AUCs, especially when excluding L2 readers, should be interpreted cautiously due to the limited number of participants.
The observed drop in performance when including L2 readers highlights population heterogeneity as an important factor.

Overall, the results demonstrate that topological features derived from fixation sequences are a viable and effective addition to dyslexia detection pipelines, and likely to be useful in analysis of eye-tracking data in general.
Future work should focus on validating these findings on larger and more diverse data sets as well as on tasks other than dyslexia detection, and on interpreting the topological features to better connect them to cognitive theories of reading and clinical practice.


\bibliographystyle{plainnat}
\bibliography{bib}


\appendix\setcounter{figure}{0}
\renewcommand{\thefigure}{\thesection.\arabic{figure}}
\setcounter{table}{0}
\renewcommand{\thetable}{\thesection.\arabic{table}}

\section{Persistent Homology of Time Series}\label{appendix:persistent_homology_of_time_series}

In this section, we review persistent homology in the context of time series.
In doing so, we will limit ourselves to those aspects of persistent homology that are relevant to the setting of this paper, and we refer the reader to~\citet{edelsbrunner_computational_topology,ghrist_elementary_applied_topology,schnider_lecture_notes_25} for details.

In the present setting, we mean by a \emph{time series} a piecewise linear function \(T\colon I\to\RR\), where \(I\subseteq\RR\) is a closed interval.
Furthermore, we require that \(T\) have only finitely many points of non-linearity in its domain.
Note that any such time series is univariate by definition.
In practice, a time series is specified by the finite sequence of tuples \(\left\{(t_{i}, x_{i})\right\}_{i=1}^{n}\subseteq\RR^{2}\), where \(x_{i}\coloneqq T(t_{i})\), \(I=[t_{1},t_{n}]\), and \(t_{2}<t_{3}<\cdots<t_{n-1}\) are the points of non-linearity of \(T\) in the interior of \(I\).
The \emph{graph} of a time series is defined as \(\Gamma(T)\coloneqq\left\{(t,T(t))\Mid t\in I\right\}\subseteq\RR^{2}\).

Given a time series \(T\colon I\to\RR\), a function \(f\colon\Gamma(T)\to\RR\) and a number \(h\in\RR\), the \emph{\(f\)-sublevel set of \(T\) at \(h\)} is defined as \[
  T_{f,h}\coloneqq \left\{(t, T(t))\Mid f(t, T(t))\leq h, t\in I\right\}\subseteq\Gamma(T).
\]
Symmetrically, the \emph{\(f\)-superlevel set of \(T\) at \(h\)} is defined as \[
  T^{f,h}\coloneqq \left\{(t, T(t))\Mid f(t, T(t))\geq h, t\in I\right\}\subseteq\Gamma(T).
\]
In words, \(T_{f,h}\) (resp. \(T^{f,h}\)) consists of that portion of the graph of \(T\) on which \(f\) attains values of at most (resp. at least) \(h\).
We point out that, in practice, the function \(f\) is often given as \(f=g\vert_{\Gamma(T)}\) for some function \(g\colon\RR^{2}\to\RR\).
For ease of exposition, we will assume now that \(f\) is continuous and generic in the sense that no two local extrema of \(f\) map to the same value under \(f\).
Observe that we have \(T_{f,h}=\varnothing\) for small enough values of \(h\).
Indeed, this is true for any value of \(h\) that is strictly smaller than the global minimum of \(f\), which we denote by \(\check{f}\).
Similarly, we have \(T_{f,h}=\Gamma(T)\) for any \(h\geq\hat{f}\), where \(\hat{f}\) is the largest value that \(f\) attains on \(\Gamma(T)\).
Moreover, we have that \(T_{f,h_{1}}\subseteq T_{f,h_{2}}\) whenever \(h_{1}\leq h_{2}\).
Put together, \(\left\{T_{f,h}\right\}_{h=-\infty}^{+\infty}\) forms a sequence of nested subsets of the graph of \(T\) that starts out as the empty set and eventually equals the entire graph \(\Gamma(T)\).
This means that \(\left\{T_{f,h}\right\}_{h=-\infty}^{+\infty}\) forms a \emph{filtration} of \(\Gamma(T)\).
If \(h\) is set to some value smaller than \(\check{f}\) and starts increasing, \(T_{f,h}\) ceases to be empty as soon as \(h\) passes the value \(\check{f}\) from below.
At that value of \(h\), the number of connected components of \(T_{f,h}\) changes from zero to one.\footnote{
  Here and in the following, a connected component is defined as one connected ``piece''; in Example~\ref{exp:main_example}, connected components are the line segments constituting the portion of the graph lying below the horizontal line.}
This logic generalizes from the global minimum value \(\check{f}\) of \(f\) to any local minimum of \(f\): whenever the parameter \(h\) passes the value of a local minimum of \(f\) from below, a new connected component of \(T_{f,h}\) is \emph{born}.
Conversely, whenever \(h\) passes the value of a local maximum \(f\) from below, the number of connected components of \(T_{f,h}\) decreases by one, since in this case an existing connected component \emph{dies} as it is absorbed by one that was born earlier.
By a similar line of argument, we see that the sequence \(\left\{T^{f,h}\right\}_{h=+\infty}^{-\infty}\) forms a filtration of \(\Gamma(T)\) (as \(h\) \emph{de}creases from \(+\infty\) to \(-\infty\)).
The number of connected components of \(T^{f,h}\) increases by one whenever \(h\) passes a local maximum of \(f\) from above, and it decreases by one if it passes a local minimum of \(f\) from above.

\emph{Persistent homology} keeps track of the evolution of the connected components of the \(f\)-sublevel and \(f\)-superlevel sets of a time series that we described in the previous paragraph by means of certain diagrams.
Namely, given a time series \(T\colon I\to\RR\), and some function \(f\colon\Gamma(T)\to\RR\), the thresholding parameter \(h\) is first set to \(-\infty\) and then increased up to \(+\infty\).
In this first phase, we record the births and deaths of connected components of \(T_{f,h}\) by means of a tuple \((b, d)\) for each connected component, where \(b\) and \(d\) denote the birth and death of a component, respectively.
Once \(h\) reaches \(+\infty\), the process is reversed.
That is, we now let \(h\) decrease back to \(-\infty\) in this second phase and record (birth, death)-tuples along the way as before.
Note that the connected component corresponding to \(\check{f}\) (the one that has the smallest birth time of all) never dies in the first phase.
While its death time is thus technically \(d=+\infty\), we treat the first birth in the second phase as its death, and hence this component has death time \(d=\hat{f}\).\footnote{
  In technical terms, this choice stems from the fact that the connected component with birth time \(\check{f}\) dies in relative homology at time \(\hat{f}\) during the second phase of the filtration.
}
Consequently, we discard the event where \(h\) passes \(\hat{f}\) from above from the set of birth events of the second phase.
The resulting collection of all (birth, death)-tuples is recorded in three \emph{persistence diagrams}, each of whose horizontal and vertical axes measure birth and death time, respectively.
The first and second of these diagrams are populated with the (birth, death)-tuples stemming from the first and second phase, respectively, while the third diagram is populated with the single (birth, death)-tuple equaling \((\check{f},\hat{f})\).
Moreover, each of the diagrams is endowed with the diagonal indicating the region where \(b=d\), the idea being that the further a point is from this diagonal, the longer the lifetime (defined as the difference between death time and birth time) of its corresponding connected component is.
These diagrams are called the \emph{ordinary}, \emph{relative} and \emph{essential diagram} of \(f\), respectively.
Note that in the first phase, the death time of any connected component is larger than its birth time, whereas in the second phase the opposite is true.
Hence all points in the ordinary and relative diagram are above and below the diagonal, respectively.
The collection of the ordinary, relative and essential diagrams is referred to as \emph{extended persistent homology} of a time series.
In contrast, \emph{ordinary persistent homology} of a time series is the result that is obtained from stopping the above process after the first phase, which produces just the ordinary diagram.
In that diagram, the point corresponding to the component with birth time \(b=\check{f}\) is either omitted, or recorded on a dashed line atop the rest of the points that serves as a placeholder for infinite death time.

We conclude this section by providing an example of an extended persistence diagram (see Example~\ref{exp:ordinary_persistence_diagram} for an example of an ordinary persistence diagram).

\begin{figure*}
  \centering
  \def\svgscale{1}
  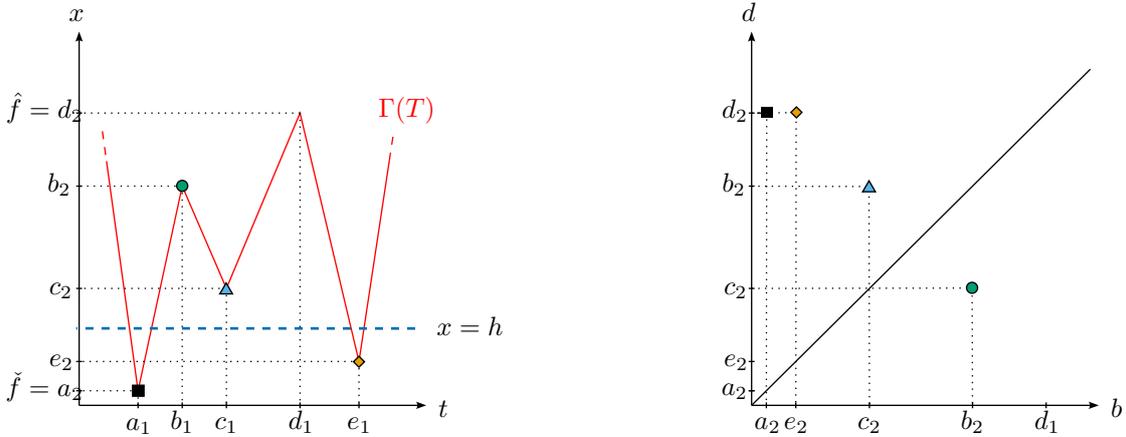
  \caption{Graph of a time series \(T\) with the horizontal filtration indicated by the line \(x=h\) (left), and the combination of the essential, ordinary and relative persistence diagrams (right).}
  \label{fig:extended_persistence}
\end{figure*}

\begin{example}[Extended persistence diagram]\label{exp:extended_persistence_diagram}
  Let \(T\colon I\to\RR\) be a time series and consider the function \(f\colon\Gamma(T)\to\RR\) given by projection onto the vertical axis, that is, \(f(t,T(t))=T(t)\) for \(t\in I\).
  The filtrations \(\left\{T_{f,h}\right\}_{h\in\RR}\) and \(\left\{T^{f,h}\right\}_{h\in\RR}\) may be thought of as sweeping \(\Gamma(T)\) by a horizontal line at vertical height \(h\).
  In the left part of Figure~\ref{fig:extended_persistence}, \(T_{f,h}\) (resp. \(T^{f,h}\)) is given by the portion of \(\Gamma(T)\) that lies below (resp. above) the dashed line defined by \(x=h\), and hence, for the choice of \(h\) depicted, consists of two (resp. three) connected components.
  As \(h\) passes the value \(\check{f}=a_{2}\) from below, the connected component corresponding to the local extremum of \(T\) marked by the square is born.
  By definition, the death time of this component is \(d=\hat{f}\), and we record this putting a point at \((b,d)=(\check{f},\hat{f})\) in the essential diagram of \(f\).
  Similarly, the local minimum of \(T\) at \(t=e_{1}\) gives rise to a connected component of \(T_{f,h}\) with \((b,d)=(e_{1},d_{2})\), and we record this tuple in the ordinary diagram of \(f\).
  Proceeding like this yields the persistence diagram given in the right part of Figure~\ref{fig:extended_persistence}, where we combined the essential, ordinary and relative diagram into one.
\end{example}

For future reference, we call the filtration from Example~\ref{exp:extended_persistence_diagram} the \emph{horizontal filtration} of a time series~\cite{ost_banana_trees,montesano_dynamically_maintaining_ph_of_time_series}.


\section{Persistence Diagrams and Machine Learning}\label{appendix:persistent_homology_and_ml}

In Appendix~\ref{appendix:persistent_homology_of_time_series}, we discussed how (extended or ordinary) persistent homology produces one or three diagrams associated with a time series.
In order to make these diagrams usable for machine learning, one must turn these diagrams into a format suitable for processing by machine learning algorithms, a procedure known as vectorization.
Vectorization methods for persistence diagrams abound, and we limit ourselves to reviewing that known as \emph{persistence image}, as introduced in~\citet{adams_persistence_images}.

Given a persistence diagram \(D=\left\{(b_{i},d_{i})\right\}_{i=1}^{n}\subseteq\RR^{2}\) consisting of (birth, death)-pairs, one starts out by discarding any points from \(D\) that have at least one infinite component.
One then applies the linear transformation \(\pi\colon\RR^{2}\to\RR^{2}\) given by \(\pi(x,y)\coloneqq(x, y-x)\); this is a shearing mapping the diagonal of the persistence diagram \(D\) to the horizontal axis of \(\RR^{2}\).
Define the \emph{persistence surface of \(\pi(D)\)} as
\begin{align*}
  \rho_{\pi(D)}\colon\RR^{2} & \to\RR                                                         \\
  (x,y)                      & \mapsto\sum_{d\in \pi(D)}w(d)\cdot\varphi_{d,\sigma^{2}}(x,y),
\end{align*}
where \(w\colon\RR^{2}\to\RR_{\geq 0}\) is a weighting function and \(\varphi_{d,\sigma^{2}}\) denotes the probability density function of a two-dimensional normalized Gaussian with mean \(d\in\RR^{2}\) and variance \(\sigma^{2}\in\RR_{>0}\).
In words, \(\rho_{\pi(D)}\) is obtained by placing a Gaussian at each point of the diagram \(\pi(D)\) (weighted according to the position of the point) and summing these up.
See Figure~\ref{fig:persistence_image} for an illustration of this process.

Finally, the \emph{persistence image} is obtained by discretizing a relevant subdomain of \(\rho_{\pi(D)}\) into a grid of pixels and assigning each pixel the integral of \(\rho_{\pi(D)}\) over that pixel, resulting in an \(n\times m\)-matrix, where \(n\) and \(m\) denote the height and width of the pixel grid, respectively.
The persistence image thus created depends on the parameters \(\sigma>0\) and \((n, m)\in\ZZ_{\geq 1}^{2}\); these are referred to as the \emph{bandwidth} and \emph{resolution}, respectively.
In practice, the persistence image is usually flattened into a vector of length \(nm\) before feeding it to machine learning algorithms.

We conclude this section by pointing out that, in practice, the weighting function \(w\) is often chosen to be \(w(x,y)=1\) or \(w(x,y)=y\) for \((x,y)\in\RR^{2}\).
In the former case, all points of a persistence diagram are weighted equally, whereas in the latter case, points of \(D\) are effectively weighted according to their lifetimes.


\section{Non-Horizontal Filtrations for Time Series}\label{appendix:non_horizontal_filtrations}

As discussed in Appendix~\ref{appendix:persistent_homology_of_time_series}, one can construct a filtration of the graph \(\Gamma(T)\subseteq\RR^{2}\) of a time series \(T\colon I\to\RR\) from a function \(f\colon\Gamma(T)\to\RR\).
Letting \(f\) be projection onto the vertical axis of \(\RR^{2}\) yields the well-known horizontal filtration of \(\Gamma(T)\).
Note that the horizontal filtration of a time series is sensitive only to the values of \(T\) at its local extrema, but not to the point in time at which they occur.
This motivates the definition and use of non-horizontal filtrations, that is, filtrations of \(T\) that explicitly depend on the time coordinate.

In the following, we define a general family of non-horizontal filtrations for time series, from which we will pick three explicit filtrations for our experiments.
To that end, let \(T\colon I\to\RR\) be a time series, and let \(F\colon J\to[0,1]\) be continuous and bijective, where \(J\subseteq\RR\) is some interval (possibly of infinite length).
Letting \(\hat{T}\) and \(\check{T}\) denote the global maximum and minimum values that \(T\) attains, respectively, we rescale and shift the function \(F\) to obtain
\begin{equation}\label{eq:def_of_G}
  G\coloneqq(\hat{T}-\check{T})F+\check{T}\colon J\to[\check{T},\hat{T}].
\end{equation}
By construction, the maximum and minimum values of \(G\) coincide with those of \(T\), and, moreover, \(G\) is continuous and bijective, too.
In particular, \(G\) is strictly monotonic.
Defining \(G_{h}(t)\coloneqq G(t-h)\), this implies that the graphs of \(G_{h}\) sweep out the strip \(\RR\times[\check{T},\hat{T}]\) as \(h\) ranges from \(-\infty\) to \(+\infty\).
Given a point \((t, x)\in\RR\times[\check{T},\hat{T}]\), this in turn guarantees that there exists a unique value \(h\in\RR\) such that \(G_{h}(t)=x\).
In particular, this applies to any point \((t_{0}, T(t_{0}))\in\Gamma(T)\), and we define \(f_{F}\colon\Gamma(T)\to\RR\) as the function that maps a point \((t_{0}, T(t_{0}))\in\Gamma(T)\) to the unique value of \(h\in\RR\) such that \(G_{h}(t_{0})=T(t_{0})\).
The function \(f_{F}\) thus defines a filtration of \(\Gamma(T)\), with respect to which we can compute (ordinary or extended) persistent homology of \(T\).
We call \(f_{F}\colon\Gamma(T)\to\RR\) the \emph{filtration associated to \(F\colon J\to[0,1]\)}.

\begin{example}[Associated filtration]\label{exp:associated_filtration}
  Let \(T\colon I\to\RR\) be a time series with global minimum and maximum values \(\check{T}\) and \(\hat{T}\), respectively, and consider the function defined by
  \begin{align*}
    F\colon[-1/2,1/2] & \to[0,1]       \\
    t                 & \mapsto t+1/2.
  \end{align*}
  Following the above, we have that \(G_{h}(t)=(\hat{T}-\check{T})(t+1/2-h)+\check{T}\).
  Given any point \((t_{0},T(t_{0}))\in\Gamma(T)\), one can check that \[
    G_{h}(t_{0})=T(t_{0})\Leftrightarrow h=t_{0}-\frac{T(t_{0})-\check{T}}{\hat{T}-\check{T}}+\frac{1}{2}.
  \]
  This value of \(h\) is, by definition, the filtration value \(f_{F}(t_{0},T(t_{0}))\).
  This filtration is obtained by sweeping out the strip \(\RR\times[\check{T},\hat{T}]\) with a line of slope \(\hat{T}-\check{T}\).
  Note that, unlike the horizontal filtration of \(T\), this filtration explicitly depends on the time-coordinate of a point in \(\Gamma(T)\).
\end{example}

With the above, we are now in a position to introduce the non-horizontal filtrations that we will use in our experiments.

\begin{definition}\label{def:sloped_sigmoid_arctan_filtrations}
  Let \(T\colon I\to\RR\) be a time series.
  Given a parameter \(c\in\RR\setminus\left\{0\right\}\), the \emph{sloped filtration of \(T\)} is the filtration associated to the function
  \begin{align*}
    \lambda_{c}\colon[-1/2c,1/2c] & \to[0,1]        \\
    t                             & \mapsto ct+1/2.
  \end{align*}
  The \emph{sigmoid filtration of \(T\)} is the filtration associated to the function
  \begin{align*}
    \sigma_{c}\colon\RR & \to[0,1]                       \\
    t                   & \mapsto\frac{1}{1+\exp(-4ct)}.
  \end{align*}
  The \emph{arctan filtration of \(T\)} is the filtration associated to the function
  \begin{align*}
    \tau_{c}\colon\RR & \to[0,1]                                         \\
    t                 & \mapsto\frac{1}{\pi}\arctan(c\pi t)+\frac{1}{2}.
  \end{align*}
\end{definition}

Crucially, note that the functions \(\lambda_{c}\), \(\sigma_{c}\) and \(\tau_{c}\) in the above definition are continuous and bijective, and hence are valid functions to construct an associated filtration of a time series from.

For an example illustrating the sloped filtration, see Example~\ref{exp:sloped_filtration}.
For the sigmoid and arctan filtrations of this time series, one replaces the sloped line by a curve described by the logistic and the arctan functions, respectively, and lets that curve move from the left to the right as before.

\begin{figure}
  \centering
  \def\svgscale{1}
  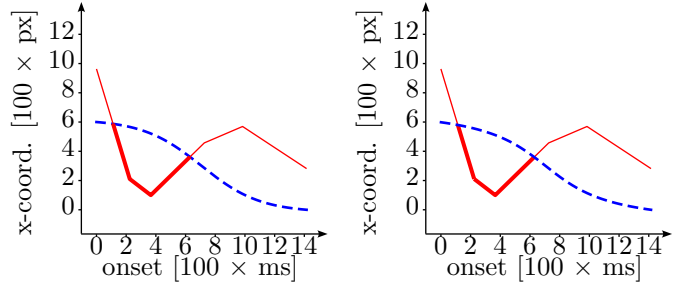
  \caption{Sweeping curves defined by the logistic (left) and arctan (right) functions.
    While these curves look very similar to the naked eye, their asymptotic behaviors are quite different.}
  \label{fig:time_series_example_sigmoid_and_arctan}
\end{figure}

The sloped filtration can be thought of as a tilted version of the usual horizontal filtration of a time series.
The idea behind the choice of the functions \(\sigma_{c}\) and \(\tau_{c}\) in Definition~\ref{def:sloped_sigmoid_arctan_filtrations} is that these functions---unlike \(\lambda_{c}\)---can ``look ahead'' when sweeping out the strip \(\RR\times[\check{T},\hat{T}]\subseteq\RR^{2}\) in which \(\Gamma(T)\) is supported; see Figure~\ref{fig:time_series_example_sigmoid_and_arctan} for illustrations of these filtrations.

While these filtrations look very similar to the naked eye, we choose to use them both in our experiments because the logistic and the arctan functions converge to their asymptotic extrema at quite different rates.

The computation from Example~\ref{exp:associated_filtration} generalizes to an explicit formula for the filtration level of a point of \(\Gamma(T)\) in terms of the function \(F\).

\begin{proposition}\label{prop:formula_explicit_filtration}
  Let \(T\colon I\to\RR\) be a time series with global minimum and maximum values \(\check{T}\) and \(\hat{T}\), respectively, and let \(F\colon J\to[0,1]\) be continuous and bijective, where \(J\subseteq\RR\) is an interval.
  Then the filtration level of a point \((t_{0},T(t_{0}))\in\Gamma(T)\) under the filtration associated to \(F\) is given by \[
    f_{F}(t_{0},T(t_{0}))=t_{0}-F^{-1}\left(\frac{T(t_{0})-\check{T}}{\hat{T}-\check{T}}\right).
  \]
\end{proposition}

\begin{proof}
  By definition, the filtration level of \((t_{0},T(t_{0}))\in\Gamma(T)\) is given as the solution to \(G_{h}(t_{0})=T(t_{0})\), which, in turn, is equivalent to \((\hat{T}-\check{T})F(t_{0}-h)+\check{T}=T(t_{0})\).
  Rearranging, this is equivalent to \(F(t_{0}-h)=\frac{T(t_{0})-\check{T}}{\hat{T}-\check{T}}\).
  Note that \(\frac{T(t_{0})-\check{T}}{\hat{T}-\check{T}}\in[0,1]\), and hence the preceding equation is equivalent to \(h=t_{0}-F^{-1}\left(\frac{T(t_{0})-\check{T}}{\hat{T}-\check{T}}\right)\), as desired.
\end{proof}

This proposition allows us to provide explicit formulas for the sloped, sigmoid and arctan filtration of a time series.

\begin{corollary}\label{cor:formulas_sloped_sigmoid_arctan}
  Let \(T\colon I\to\RR\) be a time series with global minimum and maximum values \(\check{T}\) and \(\hat{T}\), respectively, and let \((t_{0},T(t_{0}))\in\Gamma(T)\).
  Given a parameter \(c\in\RR\setminus\left\{0\right\}\), the sloped, sigmoid and arctan filtrations of \(T\) are given by \[
    f_{\lambda_{c}}(t_{0},T(t_{0}))=t_{0}-\frac{1}{c}\left(\frac{T(t_{0})-\check{T}}{\hat{T}-\check{T}}-\frac{1}{2}\right),
  \] \[
    f_{\sigma_{c}}(t_{0},T(t_{0}))=t_{0}+\frac{1}{4c}\log\left(\frac{\hat{T}-T(t_{0})}{T(t_{0})-\check{T}}\right),
  \] and
  \[
    f_{\tau_{c}}(t_{0},T(t_{0}))=t_{0}-\frac{1}{c\pi}\tan\left(\pi\left(\frac{T(t_{0})-\check{T}}{\hat{T}-\check{T}}-\frac{1}{2}\right)\right),
  \] respectively.
\end{corollary}

\begin{proof}
  The sloped filtration is, by definition, the one associated to the function \(\lambda_{c}(t)=ct+1/2\), \(t\in[-1/2c,1/2c]\), whose inverse is given by \(\lambda_{c}^{-1}(x)=\frac{1}{c}(x-1/2)\), \(x\in[0,1]\).
  Hence, by Proposition~\ref{prop:formula_explicit_filtration}, the filtration is given by \[
    f_{\lambda_{c}}(t_{0},T(t_{0}))=t_{0}-\frac{1}{c}\left(\frac{T(t_{0})-\check{T}}{\hat{T}-\check{T}}-\frac{1}{2}\right),
  \] as claimed.

  For the sigmoid filtration, note that the inverse of \(\sigma_{c}\) is given by \(\sigma_{c}^{-1}(x)=-\frac{1}{4c}\log\left(\frac{1}{x}-1\right)\), \(x\in[0,1]\).
  From this, we obtain that \[
    f_{\sigma_{c}}(t_{0},T(t_{0}))=t_{0}+\frac{1}{4c}\log\left(\frac{\hat{T}-\check{T}}{T(t_{0})-\check{T}}-1\right).
  \]
  Rearranging the argument of the logarithm yields the claimed formula.

  Finally, for the arctan filtration, note that the inverse of \(\tau_{c}\) is given by \(\tau_{c}^{-1}(x)=\frac{1}{c\pi}\tan\left(\pi(x-1/2)\right)\), \(x\in[0,1]\).
  From this, we obtain that \[
    f_{\tau_{c}}(t_{0},T(t_{0}))=t_{0}-\frac{1}{c\pi}\tan\left(\pi\left(\frac{T(t_{0})-\check{T}}{\hat{T}-\check{T}}-\frac{1}{2}\right)\right),
  \] as claimed.
\end{proof}

We conclude this section with a couple of remarks.
\begin{enumerate}
  \item We have that \(\lambda_{c}'(0)=\sigma_{c}'(0)=\tau_{c}'(0)=c\), so that \(c\) can be regarded as a parameter that controls the slope of these functions in a consistent manner.
        In practice, the parameter \(c\in\RR\setminus\left\{0\right\}\) is a hyperparameter that can be tuned according to the problem setting.
  \item As a consequence of Corollary~\ref{cor:formulas_sloped_sigmoid_arctan}, we have that \(f_{\sigma_{c}}(t,T(t))=f_{\tau_{c}}(t,T(t))=\pm\infty\) whenever \(t\in I\) is such that \(T(t)\in\{\check{T},\hat{T}\}\).
        To avoid infinite filtration values in practice, we ``pad'' the range of the time series \(T\) by a percentage of its actual range.
        This amounts to replacing the values of \(\check{T}\) and \(\hat{T}\) in Equation~\eqref{eq:def_of_G} by \(\check{T}-\varepsilon(\hat{T}-\check{T})\) and \(\hat{T}+\varepsilon(\hat{T}-\check{T})\), respectively, for some small value of \(\varepsilon>0\).
\end{enumerate}


\section{Experiment Results}\label{appendix:experiment_results}

In Table~\ref{table:results_roc_auc}, we report for each model the mean and standard deviation of the ROC AUC across all \(k\) folds (where \(k=5\) for trial-level aggregation and \(k=10\) for reader-level aggregation).

\input{tables/table_results_roc_auc}

We denote by ``TSH'' and ``BL'' the models using time series homology and the baseline models, respectively.
Moreover, we use subscripts to denote the filtration used in the topological models, and to denote the name of the baseline models.
For example, the model name ``BL\textsubscript{Bjö}+TSH\textsubscript{horizontal}'' denotes the hybrid model combining the features obtained via the baseline from~\citet{bjornsdottir2023dyslexia} with those obtained via the TSH-model using the horizontal filtration.
For the hybrid and TSH-models, we subdivide the result according to use of ordinary and extended persistence.
We highlight in bold the best performing model in each of the four settings resulting from including/excluding L2-readers and aggregating on trial/reader-level.


\end{document}

%% file: tables/table_fixation_sequence.tex
\begin{table}
  \caption{An example of a fixation sequence.}
  \label{table:fixation_sequence}
  \centering
  \begin{tabular}{rrr}
    \toprule
    Onset [ms] & \(x\)-coordinate [px] & \(y\)-coordinate [px] \\
    \midrule
    0          & 963.1                 & 533.0                 \\
    244        & 210.7                 & 95.4                  \\
    366        & 100.0                 & 130.8                 \\
    726        & 457.2                 & 124.3                 \\
    984        & 569.6                 & 120.7                 \\
    1415       & 281.6                 & 136.6                 \\
    \bottomrule
  \end{tabular}
\end{table}

%% file: figures/time_series_example.pdf_tex
\begingroup%
  \makeatletter%
  \providecommand\color[2][]{%
    \errmessage{(Inkscape) Color is used for the text in Inkscape, but the package 'color.sty' is not loaded}%
    \renewcommand\color[2][]{}%
  }%
  \providecommand\transparent[1]{%
    \errmessage{(Inkscape) Transparency is used (non-zero) for the text in Inkscape, but the package 'transparent.sty' is not loaded}%
    \renewcommand\transparent[1]{}%
  }%
  \providecommand\rotatebox[2]{#2}%
  \newcommand*\fsize{\dimexpr\f@size pt\relax}%
  \newcommand*\lineheight[1]{\fontsize{\fsize}{#1\fsize}\selectfont}%
  \ifx\svgwidth\undefined%
    \setlength{\unitlength}{506.29501343bp}%
    \ifx\svgscale\undefined%
      \relax%
    \else%
      \setlength{\unitlength}{\unitlength * \real{\svgscale}}%
    \fi%
  \else%
    \setlength{\unitlength}{\svgwidth}%
  \fi%
  \global\let\svgwidth\undefined%
  \global\let\svgscale\undefined%
  \makeatother%
  \begin{picture}(1,0.23248929)%
    \lineheight{1}%
    \setlength\tabcolsep{0pt}%
    \put(0,0){\includegraphics[width=\unitlength,page=1]{time_series_example.pdf}}%
    \put(0.05368082,0.02326765){\color[rgb]{0,0,0}\makebox(0,0)[t]{\lineheight{1.25}\smash{\begin{tabular}[t]{c}$0$\end{tabular}}}}%
    \put(0.07784395,0.02326765){\color[rgb]{0,0,0}\makebox(0,0)[t]{\lineheight{1.25}\smash{\begin{tabular}[t]{c}$2$\end{tabular}}}}%
    \put(0.102007,0.02326765){\color[rgb]{0,0,0}\makebox(0,0)[t]{\lineheight{1.25}\smash{\begin{tabular}[t]{c}$4$\end{tabular}}}}%
    \put(0.12617002,0.02326765){\color[rgb]{0,0,0}\makebox(0,0)[t]{\lineheight{1.25}\smash{\begin{tabular}[t]{c}$6$\end{tabular}}}}%
    \put(0.15033313,0.02326765){\color[rgb]{0,0,0}\makebox(0,0)[t]{\lineheight{1.25}\smash{\begin{tabular}[t]{c}$8$\end{tabular}}}}%
    \put(0.17449621,0.02326765){\color[rgb]{0,0,0}\makebox(0,0)[t]{\lineheight{1.25}\smash{\begin{tabular}[t]{c}$10$\end{tabular}}}}%
    \put(0.19865933,0.02326765){\color[rgb]{0,0,0}\makebox(0,0)[t]{\lineheight{1.25}\smash{\begin{tabular}[t]{c}$12$\end{tabular}}}}%
    \put(0.22282236,0.02326765){\color[rgb]{0,0,0}\makebox(0,0)[t]{\lineheight{1.25}\smash{\begin{tabular}[t]{c}$14$\end{tabular}}}}%
    \put(0,0){\includegraphics[width=\unitlength,page=2]{time_series_example.pdf}}%
    \put(0.03597745,0.05788032){\color[rgb]{0,0,0}\makebox(0,0)[rt]{\lineheight{1.25}\smash{\begin{tabular}[t]{r}$0$\end{tabular}}}}%
    \put(0.03597745,0.08204342){\color[rgb]{0,0,0}\makebox(0,0)[rt]{\lineheight{1.25}\smash{\begin{tabular}[t]{r}$2$\end{tabular}}}}%
    \put(0.03597745,0.10620651){\color[rgb]{0,0,0}\makebox(0,0)[rt]{\lineheight{1.25}\smash{\begin{tabular}[t]{r}$4$\end{tabular}}}}%
    \put(0.03597745,0.13036959){\color[rgb]{0,0,0}\makebox(0,0)[rt]{\lineheight{1.25}\smash{\begin{tabular}[t]{r}$6$\end{tabular}}}}%
    \put(0.03597745,0.15453267){\color[rgb]{0,0,0}\makebox(0,0)[rt]{\lineheight{1.25}\smash{\begin{tabular}[t]{r}$8$\end{tabular}}}}%
    \put(0.03597745,0.17869576){\color[rgb]{0,0,0}\makebox(0,0)[rt]{\lineheight{1.25}\smash{\begin{tabular}[t]{r}$10$\end{tabular}}}}%
    \put(0.03597745,0.20285884){\color[rgb]{0,0,0}\makebox(0,0)[rt]{\lineheight{1.25}\smash{\begin{tabular}[t]{r}$12$\end{tabular}}}}%
    \put(0.1382516,0.0031087){\color[rgb]{0,0,0}\makebox(0,0)[t]{\lineheight{1.25}\smash{\begin{tabular}[t]{c}onset [100 $\times$ ms]\end{tabular}}}}%
    \put(0.00473307,0.1341462){\color[rgb]{0,0,0}\rotatebox{90}{\makebox(0,0)[t]{\lineheight{1.25}\smash{\begin{tabular}[t]{c}x-coord. [100 $\times$ px]\end{tabular}}}}}%
    \put(0,0){\includegraphics[width=\unitlength,page=3]{time_series_example.pdf}}%
    \put(0.30665519,0.02133618){\color[rgb]{0,0,0}\makebox(0,0)[t]{\lineheight{1.25}\smash{\begin{tabular}[t]{c}$0$\end{tabular}}}}%
    \put(0.3308183,0.02133618){\color[rgb]{0,0,0}\makebox(0,0)[t]{\lineheight{1.25}\smash{\begin{tabular}[t]{c}$2$\end{tabular}}}}%
    \put(0.35498142,0.02133618){\color[rgb]{0,0,0}\makebox(0,0)[t]{\lineheight{1.25}\smash{\begin{tabular}[t]{c}$4$\end{tabular}}}}%
    \put(0.37914445,0.02133618){\color[rgb]{0,0,0}\makebox(0,0)[t]{\lineheight{1.25}\smash{\begin{tabular}[t]{c}$6$\end{tabular}}}}%
    \put(0.40330749,0.02133618){\color[rgb]{0,0,0}\makebox(0,0)[t]{\lineheight{1.25}\smash{\begin{tabular}[t]{c}$8$\end{tabular}}}}%
    \put(0.42747061,0.02133618){\color[rgb]{0,0,0}\makebox(0,0)[t]{\lineheight{1.25}\smash{\begin{tabular}[t]{c}$10$\end{tabular}}}}%
    \put(0.45163372,0.02133618){\color[rgb]{0,0,0}\makebox(0,0)[t]{\lineheight{1.25}\smash{\begin{tabular}[t]{c}$12$\end{tabular}}}}%
    \put(0.47579675,0.02133618){\color[rgb]{0,0,0}\makebox(0,0)[t]{\lineheight{1.25}\smash{\begin{tabular}[t]{c}$14$\end{tabular}}}}%
    \put(0,0){\includegraphics[width=\unitlength,page=4]{time_series_example.pdf}}%
    \put(0.28895187,0.05788032){\color[rgb]{0,0,0}\makebox(0,0)[rt]{\lineheight{1.25}\smash{\begin{tabular}[t]{r}$0$\end{tabular}}}}%
    \put(0.28895187,0.08204342){\color[rgb]{0,0,0}\makebox(0,0)[rt]{\lineheight{1.25}\smash{\begin{tabular}[t]{r}$2$\end{tabular}}}}%
    \put(0.28895187,0.10620651){\color[rgb]{0,0,0}\makebox(0,0)[rt]{\lineheight{1.25}\smash{\begin{tabular}[t]{r}$4$\end{tabular}}}}%
    \put(0.28895187,0.13036959){\color[rgb]{0,0,0}\makebox(0,0)[rt]{\lineheight{1.25}\smash{\begin{tabular}[t]{r}$6$\end{tabular}}}}%
    \put(0.28895187,0.15453267){\color[rgb]{0,0,0}\makebox(0,0)[rt]{\lineheight{1.25}\smash{\begin{tabular}[t]{r}$8$\end{tabular}}}}%
    \put(0.28895187,0.17869576){\color[rgb]{0,0,0}\makebox(0,0)[rt]{\lineheight{1.25}\smash{\begin{tabular}[t]{r}$10$\end{tabular}}}}%
    \put(0.28895187,0.20285884){\color[rgb]{0,0,0}\makebox(0,0)[rt]{\lineheight{1.25}\smash{\begin{tabular}[t]{r}$12$\end{tabular}}}}%
    \put(0.39122596,0.00117727){\color[rgb]{0,0,0}\makebox(0,0)[t]{\lineheight{1.25}\smash{\begin{tabular}[t]{c}onset [100 $\times$ ms]\end{tabular}}}}%
    \put(0.25770745,0.13414624){\color[rgb]{0,0,0}\rotatebox{90}{\makebox(0,0)[t]{\lineheight{1.25}\smash{\begin{tabular}[t]{c}x-coord. [100 $\times$ px]\end{tabular}}}}}%
    \put(0,0){\includegraphics[width=\unitlength,page=5]{time_series_example.pdf}}%
    \put(0.4004211,0.05767386){\color[rgb]{0,0,0}\makebox(0,0)[t]{\lineheight{1.25}\smash{\begin{tabular}[t]{c}birth\end{tabular}}}}%
    \put(0,0){\includegraphics[width=\unitlength,page=6]{time_series_example.pdf}}%
    \put(0.67860495,0.08068905){\color[rgb]{0,0,0}\makebox(0,0)[t]{\lineheight{1.25}\smash{\begin{tabular}[t]{c}birth\end{tabular}}}}%
    \put(0,0){\includegraphics[width=\unitlength,page=7]{time_series_example.pdf}}%
    \put(0.55971832,0.02133618){\color[rgb]{0,0,0}\makebox(0,0)[t]{\lineheight{1.25}\smash{\begin{tabular}[t]{c}$0$\end{tabular}}}}%
    \put(0.58388148,0.02133618){\color[rgb]{0,0,0}\makebox(0,0)[t]{\lineheight{1.25}\smash{\begin{tabular}[t]{c}$2$\end{tabular}}}}%
    \put(0.60804455,0.02133618){\color[rgb]{0,0,0}\makebox(0,0)[t]{\lineheight{1.25}\smash{\begin{tabular}[t]{c}$4$\end{tabular}}}}%
    \put(0.63220762,0.02133618){\color[rgb]{0,0,0}\makebox(0,0)[t]{\lineheight{1.25}\smash{\begin{tabular}[t]{c}$6$\end{tabular}}}}%
    \put(0.65637069,0.02133618){\color[rgb]{0,0,0}\makebox(0,0)[t]{\lineheight{1.25}\smash{\begin{tabular}[t]{c}$8$\end{tabular}}}}%
    \put(0.68053385,0.02133618){\color[rgb]{0,0,0}\makebox(0,0)[t]{\lineheight{1.25}\smash{\begin{tabular}[t]{c}$10$\end{tabular}}}}%
    \put(0.70469683,0.02133618){\color[rgb]{0,0,0}\makebox(0,0)[t]{\lineheight{1.25}\smash{\begin{tabular}[t]{c}$12$\end{tabular}}}}%
    \put(0.7288599,0.02133618){\color[rgb]{0,0,0}\makebox(0,0)[t]{\lineheight{1.25}\smash{\begin{tabular}[t]{c}$14$\end{tabular}}}}%
    \put(0,0){\includegraphics[width=\unitlength,page=8]{time_series_example.pdf}}%
    \put(0.54201509,0.05788032){\color[rgb]{0,0,0}\makebox(0,0)[rt]{\lineheight{1.25}\smash{\begin{tabular}[t]{r}$0$\end{tabular}}}}%
    \put(0.54201509,0.08204342){\color[rgb]{0,0,0}\makebox(0,0)[rt]{\lineheight{1.25}\smash{\begin{tabular}[t]{r}$2$\end{tabular}}}}%
    \put(0.54201509,0.10620651){\color[rgb]{0,0,0}\makebox(0,0)[rt]{\lineheight{1.25}\smash{\begin{tabular}[t]{r}$4$\end{tabular}}}}%
    \put(0.54201509,0.13036959){\color[rgb]{0,0,0}\makebox(0,0)[rt]{\lineheight{1.25}\smash{\begin{tabular}[t]{r}$6$\end{tabular}}}}%
    \put(0.54201509,0.15453267){\color[rgb]{0,0,0}\makebox(0,0)[rt]{\lineheight{1.25}\smash{\begin{tabular}[t]{r}$8$\end{tabular}}}}%
    \put(0.54201509,0.17869576){\color[rgb]{0,0,0}\makebox(0,0)[rt]{\lineheight{1.25}\smash{\begin{tabular}[t]{r}$10$\end{tabular}}}}%
    \put(0.54201509,0.20285883){\color[rgb]{0,0,0}\makebox(0,0)[rt]{\lineheight{1.25}\smash{\begin{tabular}[t]{r}$12$\end{tabular}}}}%
    \put(0.64440986,0.00117727){\color[rgb]{0,0,0}\makebox(0,0)[t]{\lineheight{1.25}\smash{\begin{tabular}[t]{c}onset [100 $\times$ ms]\end{tabular}}}}%
    \put(0.51077058,0.13414624){\color[rgb]{0,0,0}\rotatebox{90}{\makebox(0,0)[t]{\lineheight{1.25}\smash{\begin{tabular}[t]{c}x-coord. [100 $\times$ px]\end{tabular}}}}}%
    \put(0,0){\includegraphics[width=\unitlength,page=9]{time_series_example.pdf}}%
    \put(0.81269282,0.02133618){\color[rgb]{0,0,0}\makebox(0,0)[t]{\lineheight{1.25}\smash{\begin{tabular}[t]{c}$0$\end{tabular}}}}%
    \put(0.83685589,0.02133618){\color[rgb]{0,0,0}\makebox(0,0)[t]{\lineheight{1.25}\smash{\begin{tabular}[t]{c}$2$\end{tabular}}}}%
    \put(0.86101896,0.02133618){\color[rgb]{0,0,0}\makebox(0,0)[t]{\lineheight{1.25}\smash{\begin{tabular}[t]{c}$4$\end{tabular}}}}%
    \put(0.88518195,0.02133618){\color[rgb]{0,0,0}\makebox(0,0)[t]{\lineheight{1.25}\smash{\begin{tabular}[t]{c}$6$\end{tabular}}}}%
    \put(0.90934511,0.02133618){\color[rgb]{0,0,0}\makebox(0,0)[t]{\lineheight{1.25}\smash{\begin{tabular}[t]{c}$8$\end{tabular}}}}%
    \put(0.93350818,0.02133618){\color[rgb]{0,0,0}\makebox(0,0)[t]{\lineheight{1.25}\smash{\begin{tabular}[t]{c}$10$\end{tabular}}}}%
    \put(0.95767125,0.02133618){\color[rgb]{0,0,0}\makebox(0,0)[t]{\lineheight{1.25}\smash{\begin{tabular}[t]{c}$12$\end{tabular}}}}%
    \put(0.98183432,0.02133618){\color[rgb]{0,0,0}\makebox(0,0)[t]{\lineheight{1.25}\smash{\begin{tabular}[t]{c}$14$\end{tabular}}}}%
    \put(0,0){\includegraphics[width=\unitlength,page=10]{time_series_example.pdf}}%
    \put(0.79498941,0.05788032){\color[rgb]{0,0,0}\makebox(0,0)[rt]{\lineheight{1.25}\smash{\begin{tabular}[t]{r}$0$\end{tabular}}}}%
    \put(0.79498941,0.08204342){\color[rgb]{0,0,0}\makebox(0,0)[rt]{\lineheight{1.25}\smash{\begin{tabular}[t]{r}$2$\end{tabular}}}}%
    \put(0.79498941,0.10620651){\color[rgb]{0,0,0}\makebox(0,0)[rt]{\lineheight{1.25}\smash{\begin{tabular}[t]{r}$4$\end{tabular}}}}%
    \put(0.79498941,0.13036959){\color[rgb]{0,0,0}\makebox(0,0)[rt]{\lineheight{1.25}\smash{\begin{tabular}[t]{r}$6$\end{tabular}}}}%
    \put(0.79498941,0.15453267){\color[rgb]{0,0,0}\makebox(0,0)[rt]{\lineheight{1.25}\smash{\begin{tabular}[t]{r}$8$\end{tabular}}}}%
    \put(0.79498941,0.17869576){\color[rgb]{0,0,0}\makebox(0,0)[rt]{\lineheight{1.25}\smash{\begin{tabular}[t]{r}$10$\end{tabular}}}}%
    \put(0.79498941,0.20285883){\color[rgb]{0,0,0}\makebox(0,0)[rt]{\lineheight{1.25}\smash{\begin{tabular}[t]{r}$12$\end{tabular}}}}%
    \put(0.89738436,0.00117727){\color[rgb]{0,0,0}\makebox(0,0)[t]{\lineheight{1.25}\smash{\begin{tabular}[t]{c}onset [100 $\times$ ms]\end{tabular}}}}%
    \put(0.76374508,0.13414624){\color[rgb]{0,0,0}\rotatebox{90}{\makebox(0,0)[t]{\lineheight{1.25}\smash{\begin{tabular}[t]{c}x-coord. [100 $\times$ px]\end{tabular}}}}}%
    \put(0,0){\includegraphics[width=\unitlength,page=11]{time_series_example.pdf}}%
    \put(0.93153289,0.08816554){\color[rgb]{0,0,0}\makebox(0,0)[t]{\lineheight{1.25}\smash{\begin{tabular}[t]{c}death\end{tabular}}}}%
  \end{picture}%
\endgroup%

%% file: figures/persistence_diagram_example.pdf_tex
\begingroup%
  \makeatletter%
  \providecommand\color[2][]{%
    \errmessage{(Inkscape) Color is used for the text in Inkscape, but the package 'color.sty' is not loaded}%
    \renewcommand\color[2][]{}%
  }%
  \providecommand\transparent[1]{%
    \errmessage{(Inkscape) Transparency is used (non-zero) for the text in Inkscape, but the package 'transparent.sty' is not loaded}%
    \renewcommand\transparent[1]{}%
  }%
  \providecommand\rotatebox[2]{#2}%
  \newcommand*\fsize{\dimexpr\f@size pt\relax}%
  \newcommand*\lineheight[1]{\fontsize{\fsize}{#1\fsize}\selectfont}%
  \ifx\svgwidth\undefined%
    \setlength{\unitlength}{156.74588013bp}%
    \ifx\svgscale\undefined%
      \relax%
    \else%
      \setlength{\unitlength}{\unitlength * \real{\svgscale}}%
    \fi%
  \else%
    \setlength{\unitlength}{\svgwidth}%
  \fi%
  \global\let\svgwidth\undefined%
  \global\let\svgscale\undefined%
  \makeatother%
  \begin{picture}(1,0.92427638)%
    \lineheight{1}%
    \setlength\tabcolsep{0pt}%
    \put(0,0){\includegraphics[width=\unitlength,page=1]{persistence_diagram_example.pdf}}%
    \put(0.24061596,0.05284221){\color[rgb]{0,0,0}\makebox(0,0)[t]{\lineheight{1.25}\smash{\begin{tabular}[t]{c}$0$\end{tabular}}}}%
    \put(0.36035833,0.05284221){\color[rgb]{0,0,0}\makebox(0,0)[t]{\lineheight{1.25}\smash{\begin{tabular}[t]{c}$200$\end{tabular}}}}%
    \put(0.48011917,0.05284221){\color[rgb]{0,0,0}\makebox(0,0)[t]{\lineheight{1.25}\smash{\begin{tabular}[t]{c}$400$\end{tabular}}}}%
    \put(0.59986165,0.05284221){\color[rgb]{0,0,0}\makebox(0,0)[t]{\lineheight{1.25}\smash{\begin{tabular}[t]{c}$600$\end{tabular}}}}%
    \put(0.71962238,0.05284221){\color[rgb]{0,0,0}\makebox(0,0)[t]{\lineheight{1.25}\smash{\begin{tabular}[t]{c}$800$\end{tabular}}}}%
    \put(0.83936482,0.05284221){\color[rgb]{0,0,0}\makebox(0,0)[t]{\lineheight{1.25}\smash{\begin{tabular}[t]{c}$1000$\end{tabular}}}}%
    \put(0.95912567,0.05284221){\color[rgb]{0,0,0}\makebox(0,0)[t]{\lineheight{1.25}\smash{\begin{tabular}[t]{c}$1200$\end{tabular}}}}%
    \put(0,0){\includegraphics[width=\unitlength,page=2]{persistence_diagram_example.pdf}}%
    \put(0.15660459,0.15989981){\color[rgb]{0,0,0}\makebox(0,0)[rt]{\lineheight{1.25}\smash{\begin{tabular}[t]{r}$0$\end{tabular}}}}%
    \put(0.15660459,0.27964204){\color[rgb]{0,0,0}\makebox(0,0)[rt]{\lineheight{1.25}\smash{\begin{tabular}[t]{r}$200$\end{tabular}}}}%
    \put(0.15660459,0.39940284){\color[rgb]{0,0,0}\makebox(0,0)[rt]{\lineheight{1.25}\smash{\begin{tabular}[t]{r}$400$\end{tabular}}}}%
    \put(0.15660459,0.51914525){\color[rgb]{0,0,0}\makebox(0,0)[rt]{\lineheight{1.25}\smash{\begin{tabular}[t]{r}$600$\end{tabular}}}}%
    \put(0.15660459,0.6389062){\color[rgb]{0,0,0}\makebox(0,0)[rt]{\lineheight{1.25}\smash{\begin{tabular}[t]{r}$800$\end{tabular}}}}%
    \put(0.15660459,0.75864855){\color[rgb]{0,0,0}\makebox(0,0)[rt]{\lineheight{1.25}\smash{\begin{tabular}[t]{r}$1000$\end{tabular}}}}%
    \put(0.15660459,0.8784094){\color[rgb]{0,0,0}\makebox(0,0)[rt]{\lineheight{1.25}\smash{\begin{tabular}[t]{r}$1200$\end{tabular}}}}%
    \put(0.56993064,0.00039272){\color[rgb]{0,0,0}\makebox(0,0)[t]{\lineheight{1.25}\smash{\begin{tabular}[t]{c}birth\end{tabular}}}}%
    \put(0.01971313,0.49699821){\color[rgb]{0,0,0}\rotatebox{90}{\makebox(0,0)[t]{\lineheight{1.25}\smash{\begin{tabular}[t]{c}death\end{tabular}}}}}%
    \put(0,0){\includegraphics[width=\unitlength,page=3]{persistence_diagram_example.pdf}}%
  \end{picture}%
\endgroup%

%% file: figures/persistence_image.pdf_tex
\begingroup%
  \makeatletter%
  \providecommand\color[2][]{%
    \errmessage{(Inkscape) Color is used for the text in Inkscape, but the package 'color.sty' is not loaded}%
    \renewcommand\color[2][]{}%
  }%
  \providecommand\transparent[1]{%
    \errmessage{(Inkscape) Transparency is used (non-zero) for the text in Inkscape, but the package 'transparent.sty' is not loaded}%
    \renewcommand\transparent[1]{}%
  }%
  \providecommand\rotatebox[2]{#2}%
  \newcommand*\fsize{\dimexpr\f@size pt\relax}%
  \newcommand*\lineheight[1]{\fontsize{\fsize}{#1\fsize}\selectfont}%
  \ifx\svgwidth\undefined%
    \setlength{\unitlength}{384.26293945bp}%
    \ifx\svgscale\undefined%
      \relax%
    \else%
      \setlength{\unitlength}{\unitlength * \real{\svgscale}}%
    \fi%
  \else%
    \setlength{\unitlength}{\svgwidth}%
  \fi%
  \global\let\svgwidth\undefined%
  \global\let\svgscale\undefined%
  \makeatother%
  \begin{picture}(1,0.23959679)%
    \lineheight{1}%
    \setlength\tabcolsep{0pt}%
    \put(0,0){\includegraphics[width=\unitlength,page=1]{persistence_image.pdf}}%
    \put(0.12832563,0.00011017){\color[rgb]{0.16470588,0.24705882,0.37254902}\makebox(0,0)[t]{\lineheight{1.25}\smash{\begin{tabular}[t]{c}birth\end{tabular}}}}%
    \put(0.00543731,0.13224843){\color[rgb]{0.16470588,0.24705882,0.37254902}\rotatebox{90}{\makebox(0,0)[t]{\lineheight{1.25}\smash{\begin{tabular}[t]{c}death\end{tabular}}}}}%
    \put(0,0){\includegraphics[width=\unitlength,page=2]{persistence_image.pdf}}%
    \put(0.52100437,0.00007692){\color[rgb]{0.16470588,0.24705882,0.37254902}\makebox(0,0)[t]{\lineheight{1.25}\smash{\begin{tabular}[t]{c}birth\end{tabular}}}}%
    \put(0.39809803,0.13223274){\color[rgb]{0.16470588,0.24705882,0.37254902}\rotatebox{90}{\makebox(0,0)[t]{\lineheight{1.25}\smash{\begin{tabular}[t]{c}death--birth\end{tabular}}}}}%
    \put(0,0){\includegraphics[width=\unitlength,page=3]{persistence_image.pdf}}%
    \put(0.24595961,0.14787714){\color[rgb]{0,0,0}\makebox(0,0)[lt]{\lineheight{0}\smash{\begin{tabular}[t]{l}$D\mapsto \pi(D)$\end{tabular}}}}%
    \put(0.62460659,0.14787714){\color[rgb]{0,0,0}\makebox(0,0)[lt]{\lineheight{0}\smash{\begin{tabular}[t]{l}$\pi(D)\mapsto \rho_{\pi(D)}$\end{tabular}}}}%
    \put(0,0){\includegraphics[width=\unitlength,page=4]{persistence_image.pdf}}%
  \end{picture}%
\endgroup%

%% file: figures/time_series_example_sloped.pdf_tex
\begingroup%
  \makeatletter%
  \providecommand\color[2][]{%
    \errmessage{(Inkscape) Color is used for the text in Inkscape, but the package 'color.sty' is not loaded}%
    \renewcommand\color[2][]{}%
  }%
  \providecommand\transparent[1]{%
    \errmessage{(Inkscape) Transparency is used (non-zero) for the text in Inkscape, but the package 'transparent.sty' is not loaded}%
    \renewcommand\transparent[1]{}%
  }%
  \providecommand\rotatebox[2]{#2}%
  \newcommand*\fsize{\dimexpr\f@size pt\relax}%
  \newcommand*\lineheight[1]{\fontsize{\fsize}{#1\fsize}\selectfont}%
  \ifx\svgwidth\undefined%
    \setlength{\unitlength}{241.14749146bp}%
    \ifx\svgscale\undefined%
      \relax%
    \else%
      \setlength{\unitlength}{\unitlength * \real{\svgscale}}%
    \fi%
  \else%
    \setlength{\unitlength}{\svgwidth}%
  \fi%
  \global\let\svgwidth\undefined%
  \global\let\svgscale\undefined%
  \makeatother%
  \begin{picture}(1,0.43011747)%
    \lineheight{1}%
    \setlength\tabcolsep{0pt}%
    \put(0,0){\includegraphics[width=\unitlength,page=1]{time_series_example_sloped.pdf}}%
    \put(0.10979588,0.0340393){\color[rgb]{0,0,0}\makebox(0,0)[t]{\lineheight{1.25}\smash{\begin{tabular}[t]{c}$0$\end{tabular}}}}%
    \put(0.15572137,0.0340393){\color[rgb]{0,0,0}\makebox(0,0)[t]{\lineheight{1.25}\smash{\begin{tabular}[t]{c}$2$\end{tabular}}}}%
    \put(0.20164669,0.0340393){\color[rgb]{0,0,0}\makebox(0,0)[t]{\lineheight{1.25}\smash{\begin{tabular}[t]{c}$4$\end{tabular}}}}%
    \put(0.24757206,0.0340393){\color[rgb]{0,0,0}\makebox(0,0)[t]{\lineheight{1.25}\smash{\begin{tabular}[t]{c}$6$\end{tabular}}}}%
    \put(0.29349739,0.0340393){\color[rgb]{0,0,0}\makebox(0,0)[t]{\lineheight{1.25}\smash{\begin{tabular}[t]{c}$8$\end{tabular}}}}%
    \put(0.33942275,0.0340393){\color[rgb]{0,0,0}\makebox(0,0)[t]{\lineheight{1.25}\smash{\begin{tabular}[t]{c}$10$\end{tabular}}}}%
    \put(0.38534822,0.0340393){\color[rgb]{0,0,0}\makebox(0,0)[t]{\lineheight{1.25}\smash{\begin{tabular}[t]{c}$12$\end{tabular}}}}%
    \put(0.43127362,0.0340393){\color[rgb]{0,0,0}\makebox(0,0)[t]{\lineheight{1.25}\smash{\begin{tabular}[t]{c}$14$\end{tabular}}}}%
    \put(0,0){\includegraphics[width=\unitlength,page=2]{time_series_example_sloped.pdf}}%
    \put(0.0761481,0.09696427){\color[rgb]{0,0,0}\makebox(0,0)[rt]{\lineheight{1.25}\smash{\begin{tabular}[t]{r}$0$\end{tabular}}}}%
    \put(0.0761481,0.1428588){\color[rgb]{0,0,0}\makebox(0,0)[rt]{\lineheight{1.25}\smash{\begin{tabular}[t]{r}$2$\end{tabular}}}}%
    \put(0.0761481,0.18875342){\color[rgb]{0,0,0}\makebox(0,0)[rt]{\lineheight{1.25}\smash{\begin{tabular}[t]{r}$4$\end{tabular}}}}%
    \put(0.0761481,0.23464797){\color[rgb]{0,0,0}\makebox(0,0)[rt]{\lineheight{1.25}\smash{\begin{tabular}[t]{r}$6$\end{tabular}}}}%
    \put(0.0761481,0.28054251){\color[rgb]{0,0,0}\makebox(0,0)[rt]{\lineheight{1.25}\smash{\begin{tabular}[t]{r}$8$\end{tabular}}}}%
    \put(0.0761481,0.32643704){\color[rgb]{0,0,0}\makebox(0,0)[rt]{\lineheight{1.25}\smash{\begin{tabular}[t]{r}$10$\end{tabular}}}}%
    \put(0.0761481,0.37233163){\color[rgb]{0,0,0}\makebox(0,0)[rt]{\lineheight{1.25}\smash{\begin{tabular}[t]{r}$12$\end{tabular}}}}%
    \put(0.27053473,0.00223602){\color[rgb]{0,0,0}\makebox(0,0)[t]{\lineheight{1.25}\smash{\begin{tabular}[t]{c}onset [100 $\times$ ms]\end{tabular}}}}%
    \put(0.00899588,0.24182121){\color[rgb]{0,0,0}\rotatebox{90}{\makebox(0,0)[t]{\lineheight{1.25}\smash{\begin{tabular}[t]{c}x-coord. [100 $\times$ px]\end{tabular}}}}}%
    \put(0,0){\includegraphics[width=\unitlength,page=3]{time_series_example_sloped.pdf}}%
    \put(0.64249708,0.03403932){\color[rgb]{0,0,0}\makebox(0,0)[t]{\lineheight{1.25}\smash{\begin{tabular}[t]{c}$0$\end{tabular}}}}%
    \put(0.68842251,0.03403932){\color[rgb]{0,0,0}\makebox(0,0)[t]{\lineheight{1.25}\smash{\begin{tabular}[t]{c}$2$\end{tabular}}}}%
    \put(0.73434798,0.03403932){\color[rgb]{0,0,0}\makebox(0,0)[t]{\lineheight{1.25}\smash{\begin{tabular}[t]{c}$4$\end{tabular}}}}%
    \put(0.78027322,0.03403932){\color[rgb]{0,0,0}\makebox(0,0)[t]{\lineheight{1.25}\smash{\begin{tabular}[t]{c}$6$\end{tabular}}}}%
    \put(0.82619864,0.03403932){\color[rgb]{0,0,0}\makebox(0,0)[t]{\lineheight{1.25}\smash{\begin{tabular}[t]{c}$8$\end{tabular}}}}%
    \put(0.87212407,0.03403932){\color[rgb]{0,0,0}\makebox(0,0)[t]{\lineheight{1.25}\smash{\begin{tabular}[t]{c}$10$\end{tabular}}}}%
    \put(0.91804949,0.03403932){\color[rgb]{0,0,0}\makebox(0,0)[t]{\lineheight{1.25}\smash{\begin{tabular}[t]{c}$12$\end{tabular}}}}%
    \put(0.96397482,0.03403932){\color[rgb]{0,0,0}\makebox(0,0)[t]{\lineheight{1.25}\smash{\begin{tabular}[t]{c}$14$\end{tabular}}}}%
    \put(0,0){\includegraphics[width=\unitlength,page=4]{time_series_example_sloped.pdf}}%
    \put(0.60884939,0.09696434){\color[rgb]{0,0,0}\makebox(0,0)[rt]{\lineheight{1.25}\smash{\begin{tabular}[t]{r}$0$\end{tabular}}}}%
    \put(0.60884939,0.14285885){\color[rgb]{0,0,0}\makebox(0,0)[rt]{\lineheight{1.25}\smash{\begin{tabular}[t]{r}$2$\end{tabular}}}}%
    \put(0.60884939,0.18875342){\color[rgb]{0,0,0}\makebox(0,0)[rt]{\lineheight{1.25}\smash{\begin{tabular}[t]{r}$4$\end{tabular}}}}%
    \put(0.60884939,0.23464797){\color[rgb]{0,0,0}\makebox(0,0)[rt]{\lineheight{1.25}\smash{\begin{tabular}[t]{r}$6$\end{tabular}}}}%
    \put(0.60884939,0.28054253){\color[rgb]{0,0,0}\makebox(0,0)[rt]{\lineheight{1.25}\smash{\begin{tabular}[t]{r}$8$\end{tabular}}}}%
    \put(0.60884939,0.32643706){\color[rgb]{0,0,0}\makebox(0,0)[rt]{\lineheight{1.25}\smash{\begin{tabular}[t]{r}$10$\end{tabular}}}}%
    \put(0.60884939,0.37233163){\color[rgb]{0,0,0}\makebox(0,0)[rt]{\lineheight{1.25}\smash{\begin{tabular}[t]{r}$12$\end{tabular}}}}%
    \put(0.80323588,0.00223602){\color[rgb]{0,0,0}\makebox(0,0)[t]{\lineheight{1.25}\smash{\begin{tabular}[t]{c}onset [100 $\times$ ms]\end{tabular}}}}%
    \put(0.54169714,0.24182126){\color[rgb]{0,0,0}\rotatebox{90}{\makebox(0,0)[t]{\lineheight{1.25}\smash{\begin{tabular}[t]{c}x-coord. [100 $\times$ px]\end{tabular}}}}}%
    \put(0,0){\includegraphics[width=\unitlength,page=5]{time_series_example_sloped.pdf}}%
  \end{picture}%
\endgroup%

%% file: figures/pipeline.pdf_tex
\begingroup%
  \makeatletter%
  \providecommand\color[2][]{%
    \errmessage{(Inkscape) Color is used for the text in Inkscape, but the package 'color.sty' is not loaded}%
    \renewcommand\color[2][]{}%
  }%
  \providecommand\transparent[1]{%
    \errmessage{(Inkscape) Transparency is used (non-zero) for the text in Inkscape, but the package 'transparent.sty' is not loaded}%
    \renewcommand\transparent[1]{}%
  }%
  \providecommand\rotatebox[2]{#2}%
  \newcommand*\fsize{\dimexpr\f@size pt\relax}%
  \newcommand*\lineheight[1]{\fontsize{\fsize}{#1\fsize}\selectfont}%
  \ifx\svgwidth\undefined%
    \setlength{\unitlength}{386.75001526bp}%
    \ifx\svgscale\undefined%
      \relax%
    \else%
      \setlength{\unitlength}{\unitlength * \real{\svgscale}}%
    \fi%
  \else%
    \setlength{\unitlength}{\svgwidth}%
  \fi%
  \global\let\svgwidth\undefined%
  \global\let\svgscale\undefined%
  \makeatother%
  \begin{picture}(1,0.3702327)%
    \lineheight{1}%
    \setlength\tabcolsep{0pt}%
    \put(0.10278107,0.19259965){\makebox(0,0)[t]{\lineheight{1.25}\smash{\begin{tabular}[t]{c}Fixation sequence\\$T_{j,k}$\end{tabular}}}}%
    \put(0.31399696,0.13975544){\makebox(0,0)[t]{\lineheight{1.25}\smash{\begin{tabular}[t]{c}Time series\\$T_{j,k}^{y}$\end{tabular}}}}%
    \put(0,0){\includegraphics[width=\unitlength,page=1]{pipeline.pdf}}%
    \put(0.45903734,0.2595592){\color[rgb]{0,0,0}\makebox(0,0)[t]{\lineheight{0}\smash{\begin{tabular}[t]{c}PD(s)\end{tabular}}}}%
    \put(0.59414787,0.2595592){\color[rgb]{0,0,0}\makebox(0,0)[t]{\lineheight{0}\smash{\begin{tabular}[t]{c}PI(s)\end{tabular}}}}%
    \put(0,0){\includegraphics[width=\unitlength,page=2]{pipeline.pdf}}%
    \put(0.69109807,0.2662348){\makebox(0,0)[t]{\lineheight{1.25}\smash{\begin{tabular}[t]{c}flatten\end{tabular}}}}%
    \put(0,0){\includegraphics[width=\unitlength,page=3]{pipeline.pdf}}%
    \put(0.92624413,0.17712785){\makebox(0,0)[t]{\lineheight{1.25}\smash{\begin{tabular}[t]{c}PCA\end{tabular}}}}%
    \put(0,0){\includegraphics[width=\unitlength,page=4]{pipeline.pdf}}%
    \put(0.91460874,0.05247321){\makebox(0,0)[t]{\lineheight{1.25}\smash{\begin{tabular}[t]{c}Feature vector\\$v_{j,k}^{\mathrm{TSH}}$\end{tabular}}}}%
    \put(0,0){\includegraphics[width=\unitlength,page=5]{pipeline.pdf}}%
    \put(0.79445347,0.19458841){\makebox(0,0)[t]{\lineheight{1.25}\smash{\begin{tabular}[t]{c}concat\end{tabular}}}}%
    \put(0,0){\includegraphics[width=\unitlength,page=6]{pipeline.pdf}}%
    \put(0.69109807,0.08914872){\makebox(0,0)[t]{\lineheight{1.25}\smash{\begin{tabular}[t]{c}flatten\end{tabular}}}}%
    \put(0.59414787,0.15361702){\color[rgb]{0,0,0}\makebox(0,0)[t]{\lineheight{0}\smash{\begin{tabular}[t]{c}PI(s)\end{tabular}}}}%
    \put(0.45952218,0.15281223){\makebox(0,0)[t]{\lineheight{1.25}\smash{\begin{tabular}[t]{c}PD(s)\end{tabular}}}}%
    \put(0,0){\includegraphics[width=\unitlength,page=7]{pipeline.pdf}}%
    \put(0.31448177,0.33416396){\makebox(0,0)[t]{\lineheight{1.25}\smash{\begin{tabular}[t]{c}Time series\\$T_{j,k}^{x}$\end{tabular}}}}%
    \put(0,0){\includegraphics[width=\unitlength,page=8]{pipeline.pdf}}%
  \end{picture}%
\endgroup%

%% file: tables/table_hyperparameters.tex
\begin{table}
  \caption{Search spaces used for randomized hyperparameter tuning.}
  \label{table:hyperparameter_search}
  \centering
  \newsavebox{\hyperparamtablebox}
  \sbox{\hyperparamtablebox}{%
    \begin{tabular}{lll}
      \toprule
      Model class                  & Hyperparameter                      & Domain/values                               \\
      \midrule
      \multirow{2}{*}{hybrid, TSH} & Slope \(c\)\(^\ast\)                & \([-4,-0.5]\cup[0.5,4]\)                    \\
                                   & Bandwidth \(\sigma\)\(^\dagger\)    & \([10^{-3},10^{-1}]\)                       \\
      \midrule
      \multirow{4}{*}{TSH}         & Kernel\(^\ddagger\)                 & \(\left\{\text{linear},\text{rbf}\right\}\) \\
                                   & \(C\) (linear kernel)\(^\dagger\)   & \([10^{-2}, 10^{1}]\)                       \\
                                   & \(C\) (rbf kernel)\(^\dagger\)      & \([10^{-1}, 10^{2}]\)                       \\
                                   & \(\gamma\) (rbf kernel)\(^\dagger\) & \([10^{-4}, 10^{-2}]\)                      \\
      \bottomrule
    \end{tabular}%
  }
  \usebox{\hyperparamtablebox}
  \par\smallskip
  \makebox[\wd\hyperparamtablebox][l]{%
    \begin{minipage}[t]{\wd\hyperparamtablebox}
      \footnotesize
      \(^\ast\) Sampled uniformly with respect to angle\\
      \(^\dagger\) Sampled log-uniformly\\
      \(^\ddagger\) Sampled uniformly
    \end{minipage}%
  }
\end{table}

%% file: tables/table_results_roc_auc_short.tex
\begin{table}
  \caption{Mean ROC AUC scores by aggregation level and model class. Subscripts indicate standard deviation.}
  \centering
  \label{table:results_roc_auc_short}
  \begin{tabular}{clcc}
    \toprule
                &             & \multicolumn{2}{c}{Mean ROC AUC score}                            \\
    \cmidrule(lr){3-4}
    Aggregation & Model class & Including L2                           & Excluding L2             \\
    \midrule
    \multirow{3}{*}{Trial-level}
                & Hybrid      & \(\mathbf{0.88_{0.07}}\)               & \(\mathbf{0.89_{0.12}}\) \\
                & Baseline    & \(0.85_{0.12}\)                        & \(0.87_{0.17}\)          \\
                & TSH         & \(0.81_{0.17}\)                        & \(0.87_{0.16}\)          \\
    \midrule
    \multirow{3}{*}{Reader-level}
                & Hybrid      & \(\mathbf{0.92_{0.10}}\)               & \(\mathbf{0.99_{0.02}}\) \\
                & Baseline    & \(0.89_{0.08}\)                        & \(0.97_{0.04}\)          \\
                & TSH         & \(0.88_{0.17}\)                        & \(0.95_{0.07}\)          \\
    \bottomrule
  \end{tabular}
\end{table}

%% file: figures/extended_persistence.pdf_tex
\begingroup%
  \makeatletter%
  \providecommand\color[2][]{%
    \errmessage{(Inkscape) Color is used for the text in Inkscape, but the package 'color.sty' is not loaded}%
    \renewcommand\color[2][]{}%
  }%
  \providecommand\transparent[1]{%
    \errmessage{(Inkscape) Transparency is used (non-zero) for the text in Inkscape, but the package 'transparent.sty' is not loaded}%
    \renewcommand\transparent[1]{}%
  }%
  \providecommand\rotatebox[2]{#2}%
  \newcommand*\fsize{\dimexpr\f@size pt\relax}%
  \newcommand*\lineheight[1]{\fontsize{\fsize}{#1\fsize}\selectfont}%
  \ifx\svgwidth\undefined%
    \setlength{\unitlength}{423.17254639bp}%
    \ifx\svgscale\undefined%
      \relax%
    \else%
      \setlength{\unitlength}{\unitlength * \real{\svgscale}}%
    \fi%
  \else%
    \setlength{\unitlength}{\svgwidth}%
  \fi%
  \global\let\svgwidth\undefined%
  \global\let\svgscale\undefined%
  \makeatother%
  \begin{picture}(1,0.3787466)%
    \lineheight{1}%
    \setlength\tabcolsep{0pt}%
    \put(0,0){\includegraphics[width=\unitlength,page=1]{extended_persistence.pdf}}%
    \put(0.66002135,0.00373912){\color[rgb]{0,0,0}\makebox(0,0)[lt]{\lineheight{0}\smash{\begin{tabular}[t]{l}$a_{2}$\end{tabular}}}}%
    \put(0,0){\includegraphics[width=\unitlength,page=2]{extended_persistence.pdf}}%
    \put(0.68598241,0.00373912){\color[rgb]{0,0,0}\makebox(0,0)[lt]{\lineheight{0}\smash{\begin{tabular}[t]{l}$e_{2}$\end{tabular}}}}%
    \put(0,0){\includegraphics[width=\unitlength,page=3]{extended_persistence.pdf}}%
    \put(0.75088511,0.00373912){\color[rgb]{0,0,0}\makebox(0,0)[lt]{\lineheight{0}\smash{\begin{tabular}[t]{l}$c_{2}$\end{tabular}}}}%
    \put(0,0){\includegraphics[width=\unitlength,page=4]{extended_persistence.pdf}}%
    \put(0.97415039,0.0145731){\color[rgb]{0,0,0}\makebox(0,0)[lt]{\lineheight{0}\smash{\begin{tabular}[t]{l}$b$\end{tabular}}}}%
    \put(0.64859848,0.36597893){\color[rgb]{0,0,0}\makebox(0,0)[lt]{\lineheight{0}\smash{\begin{tabular}[t]{l}$d$\end{tabular}}}}%
    \put(0,0){\includegraphics[width=\unitlength,page=5]{extended_persistence.pdf}}%
    \put(0.84174892,0.00373915){\color[rgb]{0,0,0}\makebox(0,0)[lt]{\lineheight{0}\smash{\begin{tabular}[t]{l}$b_{2}$\end{tabular}}}}%
    \put(0,0){\includegraphics[width=\unitlength,page=6]{extended_persistence.pdf}}%
    \put(0.90661619,0.00373915){\color[rgb]{0,0,0}\makebox(0,0)[lt]{\lineheight{0}\smash{\begin{tabular}[t]{l}$d_{1}$\end{tabular}}}}%
    \put(0,0){\includegraphics[width=\unitlength,page=7]{extended_persistence.pdf}}%
    \put(0.63148585,0.21477823){\color[rgb]{0,0,0}\makebox(0,0)[lt]{\lineheight{0}\smash{\begin{tabular}[t]{l}$b_{2}$\end{tabular}}}}%
    \put(0.63299504,0.12368036){\color[rgb]{0,0,0}\makebox(0,0)[lt]{\lineheight{0}\smash{\begin{tabular}[t]{l}$c_{2}$\end{tabular}}}}%
    \put(0.63299271,0.05965541){\color[rgb]{0,0,0}\makebox(0,0)[lt]{\lineheight{0}\smash{\begin{tabular}[t]{l}$e_{2}$\end{tabular}}}}%
    \put(0.6310147,0.03369433){\color[rgb]{0,0,0}\makebox(0,0)[lt]{\lineheight{0}\smash{\begin{tabular}[t]{l}$a_{2}$\end{tabular}}}}%
    \put(0,0){\includegraphics[width=\unitlength,page=8]{extended_persistence.pdf}}%
    \put(0.631328,0.2803245){\color[rgb]{0,0,0}\makebox(0,0)[lt]{\lineheight{0}\smash{\begin{tabular}[t]{l}$d_{2}$\end{tabular}}}}%
    \put(0,0){\includegraphics[width=\unitlength,page=9]{extended_persistence.pdf}}%
    \put(0.38143968,0.0145731){\color[rgb]{0,0,0}\makebox(0,0)[lt]{\lineheight{0}\smash{\begin{tabular}[t]{l}$t$\end{tabular}}}}%
    \put(0.05588781,0.36597896){\color[rgb]{0,0,0}\makebox(0,0)[lt]{\lineheight{0}\smash{\begin{tabular}[t]{l}$x$\end{tabular}}}}%
    \put(0.30096037,0.00373912){\color[rgb]{0,0,0}\makebox(0,0)[lt]{\lineheight{0}\smash{\begin{tabular}[t]{l}$e_{1}$\end{tabular}}}}%
    \put(0.24903823,0.00373912){\color[rgb]{0,0,0}\makebox(0,0)[lt]{\lineheight{0}\smash{\begin{tabular}[t]{l}$d_{1}$\end{tabular}}}}%
    \put(0.18413554,0.00373912){\color[rgb]{0,0,0}\makebox(0,0)[lt]{\lineheight{0}\smash{\begin{tabular}[t]{l}$c_{1}$\end{tabular}}}}%
    \put(0.14519391,0.00373912){\color[rgb]{0,0,0}\makebox(0,0)[lt]{\lineheight{0}\smash{\begin{tabular}[t]{l}$b_{1}$\end{tabular}}}}%
    \put(0.10587758,0.00350466){\color[rgb]{0,0,0}\makebox(0,0)[lt]{\lineheight{0}\smash{\begin{tabular}[t]{l}$a_{1}$\end{tabular}}}}%
    \put(0.00025274,0.28032459){\color[rgb]{0,0,0}\makebox(0,0)[lt]{\lineheight{0}\smash{\begin{tabular}[t]{l}$\hat{f}=d_{2}$\end{tabular}}}}%
    \put(0.03689529,0.21477827){\color[rgb]{0,0,0}\makebox(0,0)[lt]{\lineheight{0}\smash{\begin{tabular}[t]{l}$b_{2}$\end{tabular}}}}%
    \put(0.03830633,0.12368041){\color[rgb]{0,0,0}\makebox(0,0)[lt]{\lineheight{0}\smash{\begin{tabular}[t]{l}$c_{2}$\end{tabular}}}}%
    \put(0.03830401,0.05965544){\color[rgb]{0,0,0}\makebox(0,0)[lt]{\lineheight{0}\smash{\begin{tabular}[t]{l}$e_{2}$\end{tabular}}}}%
    \put(0.00009725,0.03304144){\color[rgb]{0,0,0}\makebox(0,0)[lt]{\lineheight{0}\smash{\begin{tabular}[t]{l}$\check{f}=a_{2}$\end{tabular}}}}%
    \put(0,0){\includegraphics[width=\unitlength,page=10]{extended_persistence.pdf}}%
    \put(0.37998314,0.08689814){\color[rgb]{0,0,0}\makebox(0,0)[lt]{\lineheight{0}\smash{\begin{tabular}[t]{l}$x=h$\end{tabular}}}}%
    \put(0.3285593,0.28117002){\color[rgb]{1,0,0}\makebox(0,0)[lt]{\lineheight{0}\smash{\begin{tabular}[t]{l}$\Gamma(T)$\end{tabular}}}}%
    \put(0,0){\includegraphics[width=\unitlength,page=11]{extended_persistence.pdf}}%
  \end{picture}%
\endgroup%

%% file: figures/time_series_example_sigmoid_and_arctan.pdf_tex
\begingroup%
  \makeatletter%
  \providecommand\color[2][]{%
    \errmessage{(Inkscape) Color is used for the text in Inkscape, but the package 'color.sty' is not loaded}%
    \renewcommand\color[2][]{}%
  }%
  \providecommand\transparent[1]{%
    \errmessage{(Inkscape) Transparency is used (non-zero) for the text in Inkscape, but the package 'transparent.sty' is not loaded}%
    \renewcommand\transparent[1]{}%
  }%
  \providecommand\rotatebox[2]{#2}%
  \newcommand*\fsize{\dimexpr\f@size pt\relax}%
  \newcommand*\lineheight[1]{\fontsize{\fsize}{#1\fsize}\selectfont}%
  \ifx\svgwidth\undefined%
    \setlength{\unitlength}{241.14749146bp}%
    \ifx\svgscale\undefined%
      \relax%
    \else%
      \setlength{\unitlength}{\unitlength * \real{\svgscale}}%
    \fi%
  \else%
    \setlength{\unitlength}{\svgwidth}%
  \fi%
  \global\let\svgwidth\undefined%
  \global\let\svgscale\undefined%
  \makeatother%
  \begin{picture}(1,0.42726867)%
    \lineheight{1}%
    \setlength\tabcolsep{0pt}%
    \put(0,0){\includegraphics[width=\unitlength,page=1]{time_series_example_sigmoid_and_arctan.pdf}}%
    \put(0.10979582,0.03380685){\color[rgb]{0,0,0}\makebox(0,0)[t]{\lineheight{1.25}\smash{\begin{tabular}[t]{c}$0$\end{tabular}}}}%
    \put(0.1557213,0.03380685){\color[rgb]{0,0,0}\makebox(0,0)[t]{\lineheight{1.25}\smash{\begin{tabular}[t]{c}$2$\end{tabular}}}}%
    \put(0.2016466,0.03380685){\color[rgb]{0,0,0}\makebox(0,0)[t]{\lineheight{1.25}\smash{\begin{tabular}[t]{c}$4$\end{tabular}}}}%
    \put(0.24757203,0.03380685){\color[rgb]{0,0,0}\makebox(0,0)[t]{\lineheight{1.25}\smash{\begin{tabular}[t]{c}$6$\end{tabular}}}}%
    \put(0.2934974,0.03380685){\color[rgb]{0,0,0}\makebox(0,0)[t]{\lineheight{1.25}\smash{\begin{tabular}[t]{c}$8$\end{tabular}}}}%
    \put(0.33942279,0.03380685){\color[rgb]{0,0,0}\makebox(0,0)[t]{\lineheight{1.25}\smash{\begin{tabular}[t]{c}$10$\end{tabular}}}}%
    \put(0.38534822,0.03380685){\color[rgb]{0,0,0}\makebox(0,0)[t]{\lineheight{1.25}\smash{\begin{tabular}[t]{c}$12$\end{tabular}}}}%
    \put(0.43127357,0.03380685){\color[rgb]{0,0,0}\makebox(0,0)[t]{\lineheight{1.25}\smash{\begin{tabular}[t]{c}$14$\end{tabular}}}}%
    \put(0,0){\includegraphics[width=\unitlength,page=2]{time_series_example_sigmoid_and_arctan.pdf}}%
    \put(0.07614807,0.09630201){\color[rgb]{0,0,0}\makebox(0,0)[rt]{\lineheight{1.25}\smash{\begin{tabular}[t]{r}$0$\end{tabular}}}}%
    \put(0.07614807,0.14188304){\color[rgb]{0,0,0}\makebox(0,0)[rt]{\lineheight{1.25}\smash{\begin{tabular}[t]{r}$2$\end{tabular}}}}%
    \put(0.07614807,0.18746413){\color[rgb]{0,0,0}\makebox(0,0)[rt]{\lineheight{1.25}\smash{\begin{tabular}[t]{r}$4$\end{tabular}}}}%
    \put(0.07614807,0.23304514){\color[rgb]{0,0,0}\makebox(0,0)[rt]{\lineheight{1.25}\smash{\begin{tabular}[t]{r}$6$\end{tabular}}}}%
    \put(0.07614807,0.27862624){\color[rgb]{0,0,0}\makebox(0,0)[rt]{\lineheight{1.25}\smash{\begin{tabular}[t]{r}$8$\end{tabular}}}}%
    \put(0.07614808,0.32420726){\color[rgb]{0,0,0}\makebox(0,0)[rt]{\lineheight{1.25}\smash{\begin{tabular}[t]{r}$10$\end{tabular}}}}%
    \put(0.07614808,0.36978833){\color[rgb]{0,0,0}\makebox(0,0)[rt]{\lineheight{1.25}\smash{\begin{tabular}[t]{r}$12$\end{tabular}}}}%
    \put(0.27053474,0.00222079){\color[rgb]{0,0,0}\makebox(0,0)[t]{\lineheight{1.25}\smash{\begin{tabular}[t]{c}onset [100 $\times$ ms]\end{tabular}}}}%
    \put(0.0089959,0.2401694){\color[rgb]{0,0,0}\rotatebox{90}{\makebox(0,0)[t]{\lineheight{1.25}\smash{\begin{tabular}[t]{c}x-coord. [100 $\times$ px]\end{tabular}}}}}%
    \put(0,0){\includegraphics[width=\unitlength,page=3]{time_series_example_sigmoid_and_arctan.pdf}}%
    \put(0.64249704,0.03380685){\color[rgb]{0,0,0}\makebox(0,0)[t]{\lineheight{1.25}\smash{\begin{tabular}[t]{c}$0$\end{tabular}}}}%
    \put(0.68842255,0.03380685){\color[rgb]{0,0,0}\makebox(0,0)[t]{\lineheight{1.25}\smash{\begin{tabular}[t]{c}$2$\end{tabular}}}}%
    \put(0.73434792,0.03380685){\color[rgb]{0,0,0}\makebox(0,0)[t]{\lineheight{1.25}\smash{\begin{tabular}[t]{c}$4$\end{tabular}}}}%
    \put(0.78027338,0.03380685){\color[rgb]{0,0,0}\makebox(0,0)[t]{\lineheight{1.25}\smash{\begin{tabular}[t]{c}$6$\end{tabular}}}}%
    \put(0.82619855,0.03380685){\color[rgb]{0,0,0}\makebox(0,0)[t]{\lineheight{1.25}\smash{\begin{tabular}[t]{c}$8$\end{tabular}}}}%
    \put(0.87212396,0.03380685){\color[rgb]{0,0,0}\makebox(0,0)[t]{\lineheight{1.25}\smash{\begin{tabular}[t]{c}$10$\end{tabular}}}}%
    \put(0.91804957,0.03380685){\color[rgb]{0,0,0}\makebox(0,0)[t]{\lineheight{1.25}\smash{\begin{tabular}[t]{c}$12$\end{tabular}}}}%
    \put(0.96397488,0.03380685){\color[rgb]{0,0,0}\makebox(0,0)[t]{\lineheight{1.25}\smash{\begin{tabular}[t]{c}$14$\end{tabular}}}}%
    \put(0,0){\includegraphics[width=\unitlength,page=4]{time_series_example_sigmoid_and_arctan.pdf}}%
    \put(0.6088494,0.09630201){\color[rgb]{0,0,0}\makebox(0,0)[rt]{\lineheight{1.25}\smash{\begin{tabular}[t]{r}$0$\end{tabular}}}}%
    \put(0.6088494,0.14188304){\color[rgb]{0,0,0}\makebox(0,0)[rt]{\lineheight{1.25}\smash{\begin{tabular}[t]{r}$2$\end{tabular}}}}%
    \put(0.6088494,0.1874642){\color[rgb]{0,0,0}\makebox(0,0)[rt]{\lineheight{1.25}\smash{\begin{tabular}[t]{r}$4$\end{tabular}}}}%
    \put(0.6088494,0.23304514){\color[rgb]{0,0,0}\makebox(0,0)[rt]{\lineheight{1.25}\smash{\begin{tabular}[t]{r}$6$\end{tabular}}}}%
    \put(0.6088494,0.27862624){\color[rgb]{0,0,0}\makebox(0,0)[rt]{\lineheight{1.25}\smash{\begin{tabular}[t]{r}$8$\end{tabular}}}}%
    \put(0.6088494,0.32420726){\color[rgb]{0,0,0}\makebox(0,0)[rt]{\lineheight{1.25}\smash{\begin{tabular}[t]{r}$10$\end{tabular}}}}%
    \put(0.6088494,0.36978833){\color[rgb]{0,0,0}\makebox(0,0)[rt]{\lineheight{1.25}\smash{\begin{tabular}[t]{r}$12$\end{tabular}}}}%
    \put(0.80323594,0.00222079){\color[rgb]{0,0,0}\makebox(0,0)[t]{\lineheight{1.25}\smash{\begin{tabular}[t]{c}onset [100 $\times$ ms]\end{tabular}}}}%
    \put(0.54169713,0.24016948){\color[rgb]{0,0,0}\rotatebox{90}{\makebox(0,0)[t]{\lineheight{1.25}\smash{\begin{tabular}[t]{c}x-coord. [100 $\times$ px]\end{tabular}}}}}%
    \put(0,0){\includegraphics[width=\unitlength,page=5]{time_series_example_sigmoid_and_arctan.pdf}}%
  \end{picture}%
\endgroup%

%% file: tables/table_results_roc_auc.tex
\begin{table*}
  \caption{Mean ROC AUC scores by model and aggregation level. Subscripts indicate standard deviation.}
  \centering
  \label{table:results_roc_auc}
  \begin{tabular}{clcccc}
    \toprule
     &                                                         & \multicolumn{4}{c}{Mean ROC AUC score}                                                                                             \\
    \cmidrule(lr){3-6}
     & Model name
     & \multicolumn{2}{c}{Including L2}
     & \multicolumn{2}{c}{Excluding L2}                                                                                                                                                             \\
    \cmidrule(lr){3-4}\cmidrule(lr){5-6}
     &                                                         & Ord. persistence                       & Ext. persistence                    & Ord. persistence         & Ext. persistence         \\
    \midrule
    \multirow{19}{*}{\rotatebox{90}{TRIAL-LEVEL}}
     & BL\textsubscript{Bjö}+TSH\textsubscript{horizontal}     & \(0.83_{0.11}\)                        & \(0.81_{0.12}\)                     & \(0.88_{0.13}\)          & \(0.88_{0.12}\)          \\
     & BL\textsubscript{Bjö}+TSH\textsubscript{sloped}         & \(0.82_{0.12}\)                        & \(0.81_{0.12}\)                     & \(\mathbf{0.89_{0.12}}\) & \(0.88_{0.12}\)          \\
     & BL\textsubscript{Bjö}+TSH\textsubscript{sigmoid}        & \(0.82_{0.12}\)                        & \(0.80_{0.12}\)                     & \(0.88_{0.13}\)          & \(0.88_{0.12}\)          \\
     & BL\textsubscript{Bjö}+TSH\textsubscript{arctan}         & \(0.81_{0.12}\)                        & \(0.81_{0.12}\)                     & \(0.88_{0.13}\)          & \(0.88_{0.13}\)          \\
     & BL\textsubscript{Raa-RF}+TSH\textsubscript{horizontal}  & \(0.82_{0.12}\)                        & \(0.75_{0.13}\)                     & \(0.86_{0.16}\)          & \(0.83_{0.15}\)          \\
     & BL\textsubscript{Raa-RF}+TSH\textsubscript{sloped}      & \(0.85_{0.09}\)                        & \(0.77_{0.12}\)                     & \(0.86_{0.17}\)          & \(0.83_{0.16}\)          \\
     & BL\textsubscript{Raa-RF}+TSH\textsubscript{sigmoid}     & \(0.84_{0.10}\)                        & \(0.79_{0.10}\)                     & \(0.86_{0.18}\)          & \(0.82_{0.17}\)          \\
     & BL\textsubscript{Raa-RF}+TSH\textsubscript{arctan}      & \(0.85_{0.10}\)                        & \(0.83_{0.10}\)                     & \(0.86_{0.18}\)          & \(0.84_{0.18}\)          \\
     & BL\textsubscript{Raa-SVC}+TSH\textsubscript{horizontal} & \(0.87_{0.07}\)                        & \(0.87_{0.07}\)                     & \(0.86_{0.19}\)          & \(0.88_{0.18}\)          \\
     & BL\textsubscript{Raa-SVC}+TSH\textsubscript{sloped}     & \(\mathbf{0.88_{0.08}}\)               & \(\mathbf{0.88_{0.07}}\)            & \(0.87_{0.19}\)          & \(0.88_{0.18}\)          \\
     & BL\textsubscript{Raa-SVC}+TSH\textsubscript{sigmoid}    & \(0.87_{0.08}\)                        & \(\mathbf{0.88_{0.07}}\)            & \(0.86_{0.21}\)          & \(0.86_{0.21}\)          \\
     & BL\textsubscript{Raa-SVC}+TSH\textsubscript{arctan}     & \(0.87_{0.08}\)                        & \(0.85_{0.11}\)                     & \(0.87_{0.20}\)          & \(0.87_{0.20}\)          \\
     & BL\textsubscript{Bjö}                                   & \multicolumn{2}{c}{\(0.81_{0.16}\)}    & \multicolumn{2}{c}{\(0.84_{0.21}\)}                                                       \\
     & BL\textsubscript{Raa-RF}                                & \multicolumn{2}{c}{\(0.82_{0.14}\)}    & \multicolumn{2}{c}{\(0.87_{0.17}\)}                                                       \\
     & BL\textsubscript{Raa-SVC}                               & \multicolumn{2}{c}{\(0.85_{0.12}\)}    & \multicolumn{2}{c}{\(0.87_{0.18}\)}                                                       \\
     & TSH\textsubscript{horizontal}                           & \(0.80_{0.19}\)                        & \(0.80_{0.19}\)                     & \(0.85_{0.18}\)          & \(0.85_{0.19}\)          \\
     & TSH\textsubscript{sloped}                               & \(0.80_{0.17}\)                        & \(0.81_{0.17}\)                     & \(0.87_{0.16}\)          & \(0.87_{0.16}\)          \\
     & TSH\textsubscript{sigmoid}                              & \(0.79_{0.17}\)                        & \(0.79_{0.16}\)                     & \(0.85_{0.18}\)          & \(0.85_{0.19}\)          \\
     & TSH\textsubscript{arctan}                               & \(0.76_{0.17}\)                        & \(0.77_{0.17}\)                     & \(0.83_{0.15}\)          & \(0.84_{0.16}\)          \\
    \midrule
    \multirow{19}{*}{\rotatebox{90}{READER-LEVEL}}
     & BL\textsubscript{Bjö}+TSH\textsubscript{horizontal}     & \(0.89_{0.07}\)                        & \(0.75_{0.08}\)                     & \(0.96_{0.05}\)          & \(0.87_{0.13}\)          \\
     & BL\textsubscript{Bjö}+TSH\textsubscript{sloped}         & \(0.77_{0.14}\)                        & \(0.80_{0.15}\)                     & \(0.92_{0.10}\)          & \(0.88_{0.19}\)          \\
     & BL\textsubscript{Bjö}+TSH\textsubscript{sigmoid}        & \(0.84_{0.15}\)                        & \(0.84_{0.13}\)                     & \(0.92_{0.08}\)          & \(0.94_{0.05}\)          \\
     & BL\textsubscript{Bjö}+TSH\textsubscript{arctan}         & \(0.80_{0.15}\)                        & \(0.76_{0.11}\)                     & \(0.90_{0.16}\)          & \(0.91_{0.09}\)          \\
     & BL\textsubscript{Raa-RF}+TSH\textsubscript{horizontal}  & \(0.84_{0.16}\)                        & \(0.64_{0.10}\)                     & \(0.86_{0.09}\)          & \(0.85_{0.12}\)          \\
     & BL\textsubscript{Raa-RF}+TSH\textsubscript{sloped}      & \(0.85_{0.18}\)                        & \(0.79_{0.16}\)                     & \(0.82_{0.11}\)          & \(0.89_{0.06}\)          \\
     & BL\textsubscript{Raa-RF}+TSH\textsubscript{sigmoid}     & \(0.89_{0.12}\)                        & \(0.81_{0.12}\)                     & \(0.88_{0.14}\)          & \(0.93_{0.10}\)          \\
     & BL\textsubscript{Raa-RF}+TSH\textsubscript{arctan}      & \(0.87_{0.11}\)                        & \(0.86_{0.09}\)                     & \(0.80_{0.20}\)          & \(0.94_{0.05}\)          \\
     & BL\textsubscript{Raa-SVC}+TSH\textsubscript{horizontal} & \(0.89_{0.11}\)                        & \(0.90_{0.13}\)                     & \(0.98_{0.03}\)          & \(0.97_{0.07}\)          \\
     & BL\textsubscript{Raa-SVC}+TSH\textsubscript{sloped}     & \(0.84_{0.13}\)                        & \(0.86_{0.13}\)                     & \(0.95_{0.07}\)          & \(0.95_{0.07}\)          \\
     & BL\textsubscript{Raa-SVC}+TSH\textsubscript{sigmoid}    & \(\mathbf{0.92_{0.10}}\)               & \(0.88_{0.12}\)                     & \(\mathbf{0.99_{0.02}}\) & \(\mathbf{0.99_{0.02}}\) \\
     & BL\textsubscript{Raa-SVC}+TSH\textsubscript{arctan}     & \(0.87_{0.11}\)                        & \(0.90_{0.09}\)                     & \(\mathbf{0.99_{0.02}}\) & \(\mathbf{0.99_{0.02}}\) \\
     & BL\textsubscript{Bjö}                                   & \multicolumn{2}{c}{\(0.89_{0.08}\)}    & \multicolumn{2}{c}{\(0.93_{0.08}\)}                                                       \\
     & BL\textsubscript{Raa-RF}                                & \multicolumn{2}{c}{\(0.88_{0.12}\)}    & \multicolumn{2}{c}{\(0.94_{0.08}\)}                                                       \\
     & BL\textsubscript{Raa-SVC}                               & \multicolumn{2}{c}{\(0.88_{0.10}\)}    & \multicolumn{2}{c}{\(0.97_{0.04}\)}                                                       \\
     & TSH\textsubscript{horizontal}                           & \(0.84_{0.16}\)                        & \(0.84_{0.16}\)                     & \(0.92_{0.09}\)          & \(0.87_{0.11}\)          \\
     & TSH\textsubscript{sloped}                               & \(0.81_{0.18}\)                        & \(0.84_{0.12}\)                     & \(0.95_{0.07}\)          & \(0.90_{0.12}\)          \\
     & TSH\textsubscript{sigmoid}                              & \(0.83_{0.21}\)                        & \(0.82_{0.11}\)                     & \(0.93_{0.10}\)          & \(0.92_{0.05}\)          \\
     & TSH\textsubscript{arctan}                               & \(0.88_{0.17}\)                        & \(0.83_{0.17}\)                     & \(0.80_{0.19}\)          & \(0.87_{0.11}\)          \\
    \bottomrule
  \end{tabular}
\end{table*}